\documentclass{article} 
\usepackage{iclr2026_conference,times}

\usepackage{amsmath,amsfonts,bm}

\def\eqref#1{equation~\ref{#1}}

\def\1{\bm{1}}

\renewcommand\vec[1]{\bm{#1}}

\def\vtheta{{\bm{\theta}}}
\def\va{{\bm{a}}}

\def\vf{{\bm{f}}}
\def\vg{{\bm{g}}}

\def\vs{{\bm{s}}}
\def\vt{{\bm{t}}}

\def\vx{{\bm{x}}}
\def\vy{{\bm{y}}}

\def\mA{{\bm{A}}}
\def\mB{{\bm{B}}}
\def\mC{{\bm{C}}}

\def\mH{{\bm{H}}}
\def\mI{{\bm{I}}}

\def\mX{{\bm{X}}}
\def\mY{{\bm{Y}}}

\DeclareMathAlphabet{\mathsfit}{\encodingdefault}{\sfdefault}{m}{sl}
\SetMathAlphabet{\mathsfit}{bold}{\encodingdefault}{\sfdefault}{bx}{n}

\def\gD{{\mathcal{D}}}

\def\gL{{\mathcal{L}}}

\def\gX{{\mathcal{X}}}

\def\sR{{\mathbb{R}}}

\newcommand{\E}{\mathbb{E}}

\usepackage{acronym}
\usepackage{xspace}
\usepackage{xparse}

\usepackage{scalefnt}
\newcommand\smaller[2][0.80]{{\scalefont{#1}#2}}

\NewDocumentCommand{\DeclareAcr}{s m o m}{
  \IfValueTF{#3}
    {\acrodef{#2}[#3]{#4}}
    {\acrodef{#2}{#4}}
  \IfBooleanTF{#1}
    {\expandafter\newcommand\csname #2\endcsname{\acs{#2}\xspace}}
    {\expandafter\newcommand\csname #2\endcsname{\ac{#2}\xspace}}
}

\DeclareAcr{RL}[\textsc{rl}]{reinforcement learning}
\DeclareAcr{UTD}[\textsc{utd}]{update-to-data ratio}
\DeclareAcr{frames}[\textsc{fps}]{frames-per-second}

\DeclareAcr{BN}[\textsc{bn}]{batch normalisation}
\DeclareAcr{LN}[\textsc{ln}]{layer normalisation}
\DeclareAcr{WN}[\textsc{wn}]{weight normalisation}
\DeclareAcr{ELR}[\textsc{elr}]{{effective learning rate}}

\DeclareAcr{SAC}[\textsc{sac}]{soft actor-critic}

\newcommand{\XQWN}{\textsc{crossq+wn}\xspace}
\newcommand{\XQC}{\textsc{xqc}\xspace}

\newcommand{\SIMBA}{\textsc{simba-v}\smaller{2}\xspace}
\newcommand{\DRQ}{\textsc{drq-v}\smaller{2}\xspace}
\newcommand{\MrQ}{\textsc{mrq}\xspace}
\newcommand{\BRO}{\textsc{bro}\xspace}
\newcommand{\BRC}{\textsc{brc}\xspace}
\newcommand{\CFiftyOne}{\textsc{c}\smaller{51}\xspace}

\DeclareAcr{HB}[\texttt{HB}]{HumanoidBench}
\DeclareAcr{DMC}[\texttt{DMC}]{DeepMind control suite}
\DeclareAcr{Myo}[\texttt{Myo}]{MyoSuite}
\newcommand{\mujoco}{\texttt{MuJoCo}\xspace}

\newcommand{\humanoid}{\texttt{humanoid}\xspace}

\newcommand{\NEnvsVision}{$\text{15}$\xspace}
\newcommand{\NEnvsProprio}{$\text{55}$\xspace}
\newcommand{\NEnvsTotal}{$\text{70}$\xspace}

\newcommand{\NEnvsHB}{$\text{14}$\xspace}
\newcommand{\NEnvsDMC}{$\text{25}$\xspace}
\newcommand{\NEnvsMYO}{$\text{10}$\xspace}
\newcommand{\NEnvsMUJOCO}{$\text{6}$\xspace}

\DeclareAcr{IQM}[\textsc{iqm}]{inter-quartile mean}
\DeclareAcr{SBCI}[\textsc{sbci}]{stratified bootstrapped confidence intervals}
\DeclareAcr{AUC}[\textsc{auc}]{\textit{area under the curve}}
\newcommand{\FLOPS}{\textsc{flop/s}\xspace}

\newcommand{\relu}{\texttt{ReLU}\xspace}
\renewcommand{\tanh}{\texttt{tanh}\xspace}

\newcommand{\jax}{\textsc{jax}\xspace}

\DeclareAcr{MSE}[\textsc{mse}]{mean squared error}
\DeclareAcr{CE}[\textsc{ce}]{categorical cross entropy}

\usepackage[colorlinks=true,citecolor=blue,linkcolor=blue]{hyperref}
\usepackage{url}

\usepackage{graphicx}
\usepackage{amsthm}
\usepackage{cleveref}
\usepackage{booktabs}
\usepackage{makecell}

\usepackage{wrapfig}
\usepackage{duckuments}
\usepackage[normalem]{ulem}
\usepackage[most]{tcolorbox}

\usepackage{rotating}

\usepackage{enumitem}

\usetikzlibrary{plotmarks}
\newcommand\marksymbol[2]{\tikz[#2,scale=1.2]\pgfuseplotmark{#1};}

\usepackage{amsthm}
\newtheorem{theorem}{Theorem}

\newtheorem{lemma}{Lemma}
\newtheorem{proposition}{Proposition}
\newtheorem{definition}{Definition}
\newtheorem{assumption}{Assumption}

\usepackage{xcolor}
\definecolor{dense_color}{HTML}{64AB70}
\definecolor{bn_color}{HTML}{E13971}
\definecolor{ln_color}{RGB}{68,119,170}

\title{XQC: Well-conditioned Optimization\\ Accelerates Deep Reinforcement Learning}

\author{
Daniel Palenicek$^{1,2}$ Florian Vogt$^3$ Joe Watson$^4$ Ingmar Posner$^4$ Jan Peters$^{1,2,5,6}$ \\
$^1$Technical University of Darmstadt $^2$hessian.AI $^3$University of Freiburg $^4$University of Oxford \\
$^5$German Research Center for AI (DFKI) $^6$Robotics Institute Germany (RIG) \\
\texttt{daniel.palenicek@tu-darmstadt.de}
\vspace{-2em}
}

\iclrfinalcopy 
\begin{document}

\maketitle

\begin{abstract}
Sample efficiency is a central property of effective deep reinforcement learning algorithms.
Recent work has improved this through added complexity, such as larger models, exotic network architectures, and more complex algorithms, 
which are typically motivated purely by empirical performance. 
We take a more principled approach by focusing on the optimization landscape of the critic network. Using the eigenspectrum and condition number of the critic's Hessian, we systematically investigate the impact of common architectural design decisions on training dynamics.
Our analysis reveals that a novel combination of batch normalization~(\textsc{bn}), weight normalization~(\textsc{wn}), and a distributional cross-entropy~(\textsc{ce}) loss produces condition numbers orders of magnitude smaller than baselines.
This combination also naturally bounds gradient norms, a property critical for maintaining a stable effective learning rate under non-stationary targets and bootstrapping.
Based on these insights, we introduce \textsc{xqc}: a well-motivated, sample-efficient deep actor-critic algorithm built upon soft actor-critic that embodies these optimization-aware principles.
We achieve state-of-the-art sample efficiency across \NEnvsProprio proprioception and \NEnvsVision vision-based continuous control tasks, all while using significantly fewer parameters than competing methods.
Our code is available at \href{https://danielpalenicek.github.io/projects/xqc}{danielpalenicek.github.io/projects/xqc}.
\end{abstract}

\begin{figure}[h!]
    \centering
    \vspace{-.5em}
    \vspace{-.5em}
    \includegraphics[width=\linewidth]{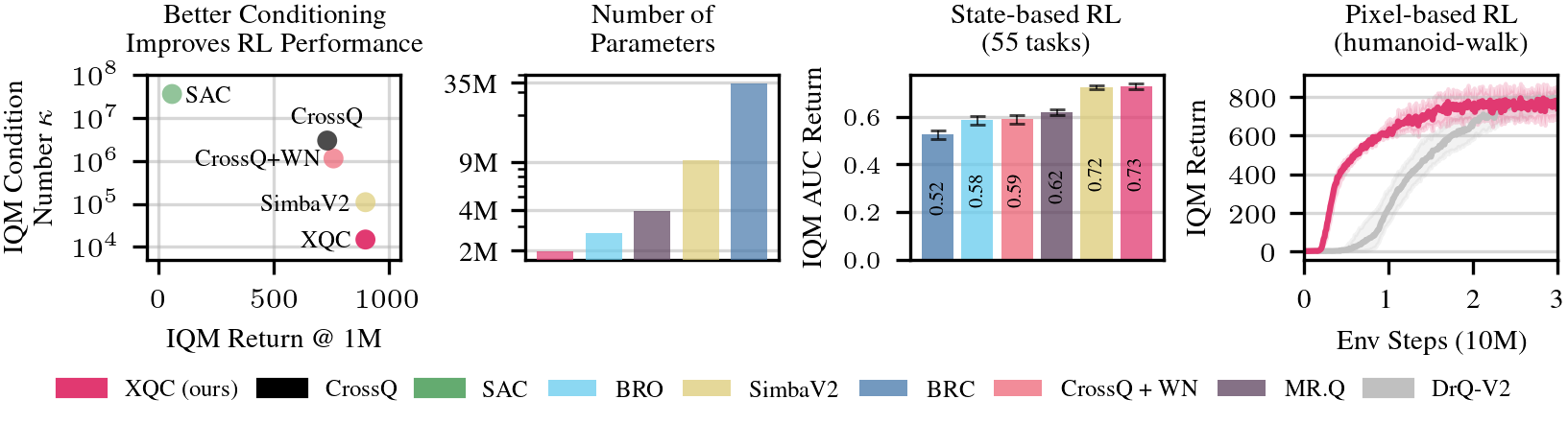}
    \vspace{-2.3em}
    \caption{
    \textbf{Well-conditioned network architectures yield state-of-the-art \RL performance.}
    Our algorithm, \XQC with a \BN and \WN-based architecture and a \CE loss, achieves competitive performance against state-of-the-art baselines across \NEnvsProprio proprioceptive continuous control tasks from four different benchmarks with a single set of hyperparameters.
    Notably, with $\sim\!4.5\times$ fewer parameters and $\sim\!5\times$ less compute in terms of \FLOPS than \SIMBA, the closest competitor.
    \XQC's efficiency carries over to \RL from pixels on \NEnvsVision vision-based \DMC tasks, significantly improving on \DRQ.
    }
    \label{fig:hero}
    \vspace{-.5em}
\end{figure}

\acresetall

\section{Introduction}
Sample efficiency remains a major challenge in deep \RL.
Methods that can learn effectively from limited interactions are crucial for applying \RL in domains such as robotics~\citep{kober2008policy,bohlinger2025gait}, where generating data on real hardware is costly and time-consuming.
Recent advances in model-free \RL have primarily been driven by a paradigm of scaling---larger networks, higher \UTD ratios, and ever-increasing computational budgets~\citep{
nikishin2022primacy,
nauman2024bigger,
lee2024simba,
palenicek2025scaling}.
These works view architectural improvements primarily as a means to scale up stably
and in turn improve sample efficiency.
While proven effective, this \textit{`bigger is better'} paradigm comes at the cost of computational efficiency and overlooks a more fundamental question:
\textit{Can we improve performance not by adding complexity, but by creating a better-conditioned optimization problem?}

To develop this principled understanding, we conduct a systematic investigation into three commonly used architectural components whose roles in \RL are often guided by heuristics.
First, we examine normalization layers. 
The \RL community frequently uses \LN~\citep{ba2016layernorm}, e.g., \citet{ball2023efficient}, in contrast to \BN~\citep{ioffe2015batch}, which until recently was thought to be problematic in the \RL setting \citep{bhatt2024crossq}.
Second, we consider \WN~\citep{lyle2024normalization,loshchilov2025ngpt,palenicek2025scaling} by periodically projecting the network's weights to the unit sphere, permitted through the normalization layers' \textit{scale invariance} property. A technique known to improve the \ELR~\citep{van2017l2}.
Lastly, we study the critic's loss function. Distributional critics using a \CE loss have grown in popularity~\citep{bellemare2017c51}. 
The conventional argument for their adoption is that modeling the full distribution of returns provides a better learning signal compared to \MSE regression~\citep{farebrotherstop}; there is evidence that this loss is easier to optimize \citep{imani2018improving}.

Through a systematic eigenvalue analysis of the critic's Hessian, we provide a principled explanation for \textit{why} different architectures outperform others.
Our analysis first shows that \BN consistently produces better-conditioned local loss landscapes than \LN during learning, with condition numbers that are orders of magnitude smaller.
Second, our investigation of the critic loss reveals that, beyond its representational advantages, the \CE loss induces a remarkably well-conditioned optimization landscape compared to the \MSE loss.
We find that the combination of \BN, \WN, and a categorical \CE loss works in synergy to dramatically improve the conditioning of the optimization problem and stabilize the \ELR, a key metric for maintaining plasticity in deep \RL.
In summary, we claim the following contributions:
\begin{enumerate}[leftmargin=2em]
    \item \XQC, a simple and efficient extension to soft actor critic, uses the powerful synergy between \BN, \WN, and a distributional critic with a \CE Bellman error loss for sample-efficient learning.
    \item A Hessian eigenvalue analysis of modern deep \RL critics, revealing the superior conditioning properties of distributional critic losses over the mean squared error.
    \item Extensive empirical validation on \NEnvsProprio proprioception and \NEnvsVision vision-based continuous control tasks, demonstrating state-of-the-art performance against more complex, larger-scale methods.
\end{enumerate}

\section{Preliminaries}
\label{sec:prelim}
This section briefly introduces the necessary background and notation for this paper.

\textbf{Deep reinforcement learning.}
In this work, we assume the standard \RL setting~\citep{Sutton1998}, where an agent attempts to learn a policy that maximizes its expected discounted return. 
Our experiments are based on the popular off-policy actor-critic algorithm \SAC~\citep{haarnoja2018sac}, where policy and critic are represented by neural networks.
A key quantity in reinforcement learning with function approximation is the Bellman error $\Delta_\vtheta$,
\begin{align}
    \Delta_\vtheta(\vs,\va) = Q(\vs,\va) - Q_\vtheta(\vs,\va),
    \quad
    Q(\vs,\va) = r(\vs,\va) + \gamma\,\E_{\vs'\sim p(\cdot\mid\vs,\va)}[V(\vs')],
\end{align}
where $Q$ and $V$ are the `soft' parametric critic and value functions, respectively.
Minimizing this Bellman error effectively is key to the success of actor-critic methods~\citep{Sutton1998}.
This error is typically minimized with the mean squared regression loss and gradient-based optimization.
The distributional \CFiftyOne algorithm~\citep{bellemare2017c51} reformulates the task as a classification problem. Instead of a scalar estimate, the function approximator outputs the logits of a categorical distribution over the full support of $Q$.
The Bellman error can then be minimized using a categorical \CE loss, which has been shown to improve performance and stability~\citep{farebrotherstop}.

\textbf{Analyzing gradient-based optimization.}
To analyze the optimization of our gradient-based updates, we consider their first- and second-order aspects.
For gradient-based optimization with parameter normalization, we must consider the effective learning rate (Definition \ref{def:elr}).

\begin{definition}
\label{def:elr}
(Effective learning rate, \ELR, \citet{van2017l2}).
For a scale-invariant function $f(\vtheta) = f(\lambda\,\vtheta), \lambda{\,>\,}0$, the `effective' learning rate $\tilde{\eta}$ for an update 
$f(\vtheta{\,+\,}\eta\,\vg(\vtheta))$ with gradients $\vg(\vtheta)$ is the learning rate when taking this scale invariance into account,
\begin{equation*}
    f(\vtheta{\,+\,}\eta\,\vg(\vtheta))
    =
    f(\tilde{\vtheta}{\,+\,}\tilde{\eta}\,\vg(\tilde{\vtheta)}),
    \quad
    \tilde{\eta} = \eta / \|\vec{\theta}\|_2,
    \quad
    \tilde{\vtheta} = \vtheta / \|\vec{\theta}\|_2.
\end{equation*}
\end{definition}
Recent work has studied \ELR in the context of loss of `plasticity' in neural networks and scaling gracefully to larger \UTD ratios in \RL \citep{lyle2024normalization,palenicek2025scaling}.

To analyze the second-order properties of the loss landscape, a local quadratic approximation
\begin{align}
    \gL(\vtheta + \delta\vtheta) \approx
    \gL(\vtheta) + 
    \nabla_\vtheta \gL(\vtheta)\delta\vtheta +
    \frac{1}{2}{\delta\vtheta}^\top \nabla^2_{\vtheta}\gL(\vtheta)\delta\vtheta,
\end{align}
illustrates the role of the Hessian $\nabla^2_{\vtheta}\gL(\vtheta)$ in characterizing the curvature of the local loss landscape, which is measured using its eigenvalues.
The $\beta$-smoothness of the objective upper-bounds the largest eigenvalue of the Hessian (Definition \ref{def:beta}), while the ratio of largest to smallest absolute values describes the condition number (Definition \ref{def:condition}).
The larger the condition number, the less effective gradient descent with a fixed learning rate will be due to the large range in curvature per dimension, as illustrated in  Figure \ref{fig:condition_number} \citep{nocedal2006numerical}.
While we use an adaptive learning rate optimizer (Adam,  \citet{kingma2014adam}), whose adaptivity helps overcome issues with ill-conditioning, the loss landscape curvature remains relevant when assessing optimization difficulty.

\begin{definition}
\label{def:beta}
($\beta$-smoothness, \citet{drusvyatskiy2020convex}).
A loss function $\gL(\vtheta)$ is said to be $\beta$-smooth if its gradient is Lipschitz continuous with constant $\beta$, i.e.,
$
    \|\nabla_\vtheta \gL(\vtheta_1) - \nabla_\vtheta \gL(\vtheta_2) \| \leq \beta\, \|\vtheta_1 - \vtheta_2 \|
$
which is equivalent to the largest eigenvalue of its Hessian being bounded by $\beta$, i.e., $\lambda_{\max}(\nabla^2_\vtheta \gL(\vtheta))\leq\beta$.
As such, the $\beta$-smoothness quantifies the maximum curvature of the landscape.
\end{definition}
In our experiments, we will look at the largest eigenvalue as a proxy for the empirical measure of $\beta$.
\begin{definition}
\label{def:condition}
(Condition number, \citet{nocedal2006numerical}).
For a normal matrix $\mH \in \sR^{d \times d}$ with eigenvalues $\lambda_1,  \dots, \lambda_d$, its condition number $\kappa$ is 
$
    \kappa(\mH) = \textstyle\max_i|\lambda_{i}| / \textstyle\min_i|\lambda_i|.
$
As a measure of sensitivity, a low condition number describes a `well-conditioned' matrix, while a high condition number describes an `ill-conditioned' matrix.
\end{definition}
For a Hessian, the condition number is used to analyze and characterize the effectiveness of gradient-based optimization algorithms \citep{nocedal2006numerical}.
Using the insights of the effective learning rate, $\beta$-smoothness, and condition numbers, we now compare the optimization landscapes of different actor-critic architectures in deep \RL.

\begin{figure}[!t]
    \vspace{-1em}
    \includegraphics[width=\linewidth]{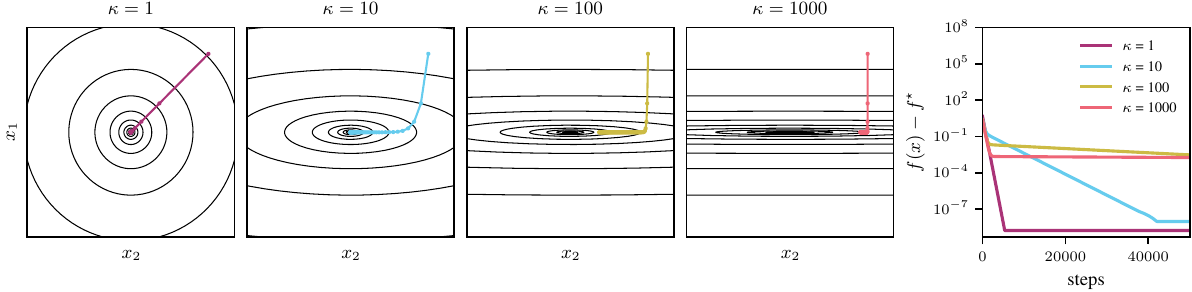}
    \vspace{-2em}
    \caption{When performing gradient-based optimization, the condition number ($\kappa$) of the objective's Hessian significantly impacts convergence. 
    We illustrate this phenomenon with a simple two-dimensional quadratic example.
    As $\kappa$ increases by an order of magnitude, gradient descent converges at a lower rate.
    We believe this phenomenon plays a similar role when learning the critic in deep reinforcement learning, where high condition numbers lead to poor sample efficiency.
    }
    \label{fig:condition_number}
\end{figure}

\section{The Optimization Landscapes of the Bellman error}\label{sec:analysis}

We seek to improve sample efficiency by enhancing the critic's optimization landscape. 
This section applies the optimization insights from Section \ref{sec:prelim} to the Bellman error minimization.
Hessian eigenvalues have previously been used to understand the benefits of batch normalization in supervised learning~\citep{ghorbani2019hessian_eigenspec}.
To our knowledge, this is the first such analysis for deep~\RL.

\subsection{An empirical investigation of critic optimization.}
\label{sec:empirical_investigation}
To quantify the impact of common architectural components on the optimization of the Bellman error, we looked at the eigenvalues of the critic's Hessian while learning the challenging \DMC \texttt{dog-trot} environment, a high-dimensional continuous control task, ensuring the findings are not artifacts of a trivial toy-task.

We systematically compare critic networks with combinations of common architectural components:
\begin{itemize}[noitemsep,topsep=0pt]
    \item Normalization strategies: \textcolor{bn_color}{\BN}, \textcolor{ln_color}{\LN}, None (\textcolor{dense_color}{Dense}).
    \item Weight projection to the unit sphere: \WN (\marksymbol{square}{black}), no \WN (\marksymbol{square*}{black}).
    \item Loss functions: \MSE (\marksymbol{triangle}{black}), \CE (\marksymbol{o}{black}).
\end{itemize}
This results in a total of 12 distinct architectural combinations, 
allowing for a thorough dissection of each component's contribution.
Per architecture, we run 5 random seeds for 1M environment steps and compute the Hessian eigenspectrum at 20 checkpoints throughout training using an efficient \jax~\citep{jax2018github} implementation of the \textit{stochastic Lanczos quadrature} algorithm~\citep{golub1969calculation,lin2016approximating}, adapted from \citet{ghorbani2019hessian_eigenspec}.

\begin{figure}[t]
    \centering
    \includegraphics[width=\linewidth]{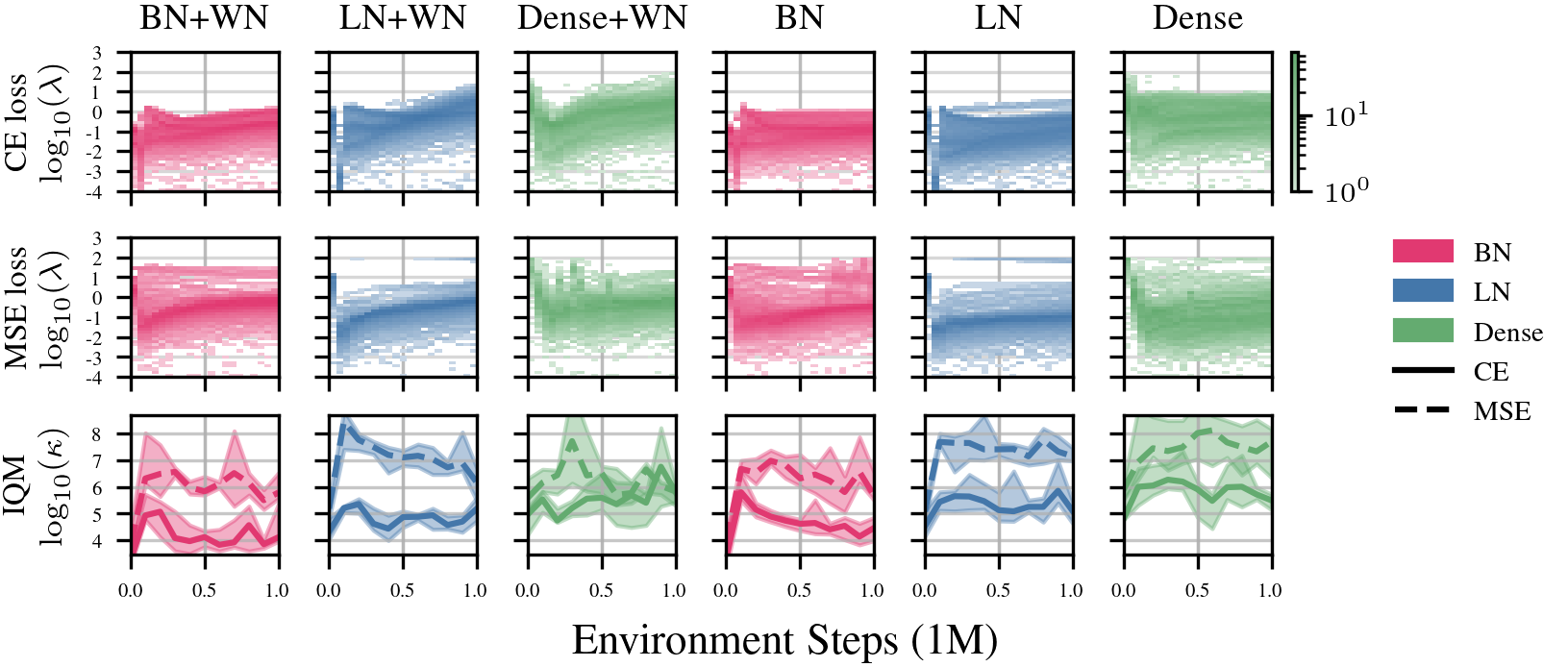}
    \caption{
        Eigenvalues and condition numbers on \texttt{dog-trot} over 5 seeds for different critic architectures during training.
        The top and middle rows show the eigenspectra of the \CE loss and \MSE loss, respectively.
        The columns correspond to different combinations of normalization layers and \WN.
        The bottom row shows the \IQM and 90\% \SBCI of the condition number $\kappa$ aggregated over five seeds for \CE and \MSE losses, respectively.
        Architectures using \BN show more compact and stable eigenspectra over the course of training with no outliers. \LN suffers from large outlier modes and includes overall larger eigenvalues.
        Similarly, the \CE loss significantly improves loss landscape conditioning over an \MSE.
    }
    \label{fig:full_spectra}
\end{figure}

\textbf{Eigenvalue analysis.}
First, we qualitatively analyze how the eigenvalues evolve during training for the different architectures.
\Cref{fig:full_spectra} (top \& middle) reveals striking differences in the curvature of the loss landscape for the different components.
Architectures employing \BN consistently produce more compact and stable eigenspectra throughout training,
with eigenvalues remaining bounded within a moderate range and free of significant outliers.
In stark contrast, \LN  architectures suffer from large, growing outlier eigenvalues, signifying sharp curvature that can destabilize training.
Similarly, the \CE loss significantly improves loss landscape conditioning over an \MSE. 
This is also reflected in the condition numbers \Cref{fig:full_spectra} (bottom), where \BN-based architectures are consistently an order of magnitude smaller and more stable than their non-\BN counterparts.

\textbf{Condition numbers and $\beta$-smoothness.}
To make the relationship between the spectral properties and performance explicit, \Cref{fig:spectrum_summary} presents the data in an aggregated form.
Each point shows aggregated results over 5 seeds, correlating an architecture's \IQM condition number, \IQM $\max(\lambda)$ and \IQM Kurtosis$(\lambda)$, respectively, over the entire course of training, with its sample-efficiency (\IQM return at 1M timesteps).
These plots show a clear and strong trend: architectures with lower condition numbers and smaller maximum eigenvalues achieve higher returns.
This trend provides compelling empirical evidence that, perhaps unsurprisingly, a smoother, better-conditioned optimization landscape is a key driver of performance in deep \RL.
The Kurtosis provides a proxy measure for outliers in the eigenspectrum, where \BN-based architectures consistently show lower Kurtosis than their \LN-based counterparts.
Furthermore, the results show that \BN, \WN, and a categorical \CE loss each independently improve the landscape's conditioning, and when combined, their synergistic effect yields the best-conditioned landscape and the highest performance.
\newpage

\begin{figure}
    \centering
    \includegraphics[width=\linewidth]{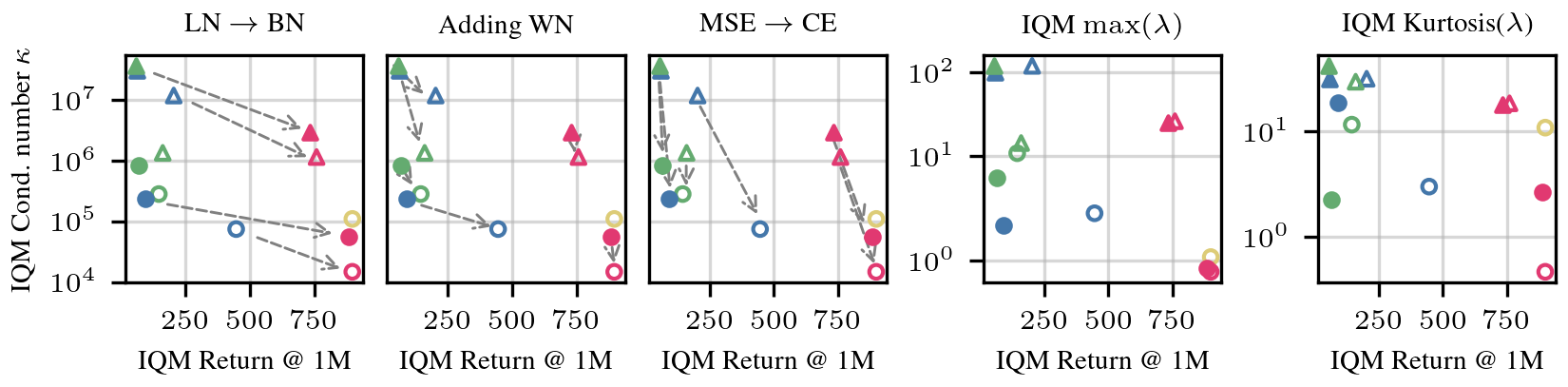}
    \vspace{-2em}
    \caption{
    The condition numbers and maximum eigenvalues against the return at 1M steps on \DMC \texttt{dog-trot}.
    Normalization strategies are color-coded \textcolor{bn_color}{\BN}, \textcolor{ln_color}{\LN}, \textcolor{dense_color}{Dense}.
    Use of \WN = empty shape~\protect\marksymbol{square}{black}, whereas no \WN is represented by a filled shape~\protect\marksymbol{square*}{black}.
    \MSE~loss~=~\protect\marksymbol{triangle}{black} and \CE~=~\protect\marksymbol{o}{black}.
    Architectures with lower condition numbers and lower maximum eigenvalues tend to have better final returns.
    Also, \BN, \WN, and the categorical \CE loss each improve the loss conditioning independently (columns 1-3).
    Combined, they result in the best conditioning and best performance~\protect\marksymbol{o}{blue}.
    For reference, we include \SIMBA~\protect\marksymbol{o}{orange} a strong baseline with a similarly low condition number.
    }
    \label{fig:spectrum_summary}
\end{figure}

\subsection{Why does cross-entropy outperform the squared error?}
\label{sec:ce_analysis}
Distributional \RL and \CFiftyOne were proposed to perform distributional regression of the returns, beyond predicting only the average value \citep{bellemare2023distributional}.
In this section, we motivate distributional losses from the optimization perspective \citep{imani2018improving} to explain the dramatic difference in condition numbers between \CE and \MSE Bellman errors in Section \ref{sec:empirical_investigation}.
We show that the \CE loss has desirable properties for optimization over the \MSE.
Firstly, Propositions \ref{prop:mse_grad} and \ref{prop:ce_grad} show that the gradient norm for the loss with respect to the predictions can only be bounded for the \CE loss.
\begin{proposition}
\label{prop:mse_grad}
The loss, $l(\vy, \hat{\vy}) = \frac{1}{2}||\vy -\hat{\vy}||_2^2$ has unbounded gradients w.r.t. $\hat{\vy}$,
\begin{align}
    ||\nabla_{\hat{\vy}} l(\vy, \hat{\vy})||_2 = ||\vy - \hat{\vy}||_2 \leq \infty,
    \quad
    \hat{\vy} = \vf_\vtheta(\vx).
\end{align}
\end{proposition}

\begin{proposition}
\label{prop:ce_grad}
The loss, $l(\vt, \hat{\vy}){\,=\,}{-}\sum_{i=1}^d t_{i}\log \hat{t}_{i},\,\hat{\vt}{\,=\,}\textrm{Softmax}(\hat{\vy})$ has \underline{bounded} gradients w.r.t. $\hat{\vy}$,
\begin{align}
    ||\nabla_{\hat{\vy}} l(\vt, \hat{\vy})||_2 = ||\vt - \textrm{Softmax}(\hat{\vy})||_2 \leq \sqrt{2},
    \quad
    {\hat{\vy} = \vf_\vtheta(\vx))}.
\end{align}
\end{proposition}
Combining Proposition \ref{prop:ce_grad} with weight normalization and a Lipschitz assumption, we can upper bound the effective gradient update (Definition \ref{def:elr}) in Theorem \ref{th:ce_elr}.
\begin{theorem}
\label{th:ce_elr}
For the cross entropy-loss $l$ and learning rate $\eta\geq 0$, for a scale-invariant function approximator $\vf_\vtheta$ which is $L_\vf$ Lipschitz continuous in the L2 norm with respect to $\vtheta$ with fixed parameter norm $||\vtheta||_2 = C$, the effective gradient update can be upper bounded as
\begin{align}
    \eta{||\vtheta||_2}^{-1}\nabla_\vtheta \l(\vt, \vf_\vtheta(\vx))
    \leq \eta\,C^{-1}\sqrt{2} L_\vf.
\end{align}
\end{theorem}
To analyze second-order properties, we must assume that we can bound the eigenvalues of the function approximator's Hessian so that weight decay can ensure the Hessian of the objective is positive definite and the smallest eigenvalue is greater than zero.
\begin{assumption}
\label{ass:hess}
We assume eigenvalue bounds for the function approximator Hessian (per output),
\begin{align*}
    \textstyle 0 \leq |\lambda_1^f| \leq |\sigma_i(\nabla^2_\vtheta f_\vtheta(\vx))| \leq |\lambda_m^f| < \infty
    \quad
    \forall\; i\in[0, m],\vx \in \gX, \vtheta \in \Theta.
\end{align*}
\end{assumption}
\begin{proposition}
\label{prop:mse_hess}
Given Assumption \ref{ass:hess}, the eigenvalues of the Hessian of the mean squared error loss with weight decay $\mu^2$, $\mu{\,\geq\,}0$ are unbounded and the condition number \underline{cannot} be upper bounded.
\end{proposition}
\begin{proposition}
\label{prop:ce_hess}
Given Assumption \ref{ass:hess}, the eigenvalues of the Hessian of the cross-entropy loss with weight decay $\mu^2$ have an upper-bounded condition number
\begin{align*}
    \textstyle \kappa\left(\nabla^2_\vtheta \gL\right) \leq 
    (4\lambda^f_m + L_\vf^2 + \epsilon)/\epsilon,
    \quad
    \epsilon \geq 0,
\end{align*}
when $\mu^2 = 2 \lambda_m^f + \epsilon$, which provides a finite upper bound when $\epsilon > 0$.
\end{proposition}

For proofs see \Cref{app:theoretical_anaylsis}.
In practice, we do not require weight decay to attain good performance, and as a result, the Hessian was observed to not always be positive definite.
Nonetheless, these results provide a formal intuition for our empirical result that \CE losses consistently report smaller condition numbers.
Our analysis directly motivates the design of our \XQC algorithm, presented next.

\section{XQC: A Simple \& Well-conditioned Actor-Critic Architecture}\label{sec:xqc}

\begin{wrapfigure}[24]{r}{0.25\textwidth}
  \vspace{-2.2em}
  \begin{center}
  \includegraphics[width=0.25\textwidth]{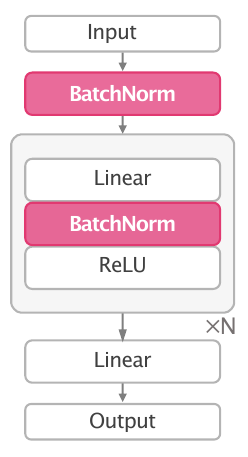}
  \end{center}
  \caption{The \XQC network architecture consists of only three standard components: Linear, \BN, and \relu for a total of 4 hidden layers.}
  \label{fig:architecture}
\end{wrapfigure}

This section presents our novel \XQC algorithm, a direct conclusion of our optimization analysis in \Cref{sec:analysis}.
It is a simple, yet powerful architecture with the purpose of improving the loss landscapes optimization behaviours, that extends the popular \SAC algorithm.
\XQC's critic architecture is motivated by a central principle: combining components that synergistically improve optimization dynamics.
We provide a complete list of hyperparameters in \Cref{sec:hyperparameters}.

\textbf{Batch normalization.}
\XQC uses \BN layers directly on the network input and after each linear layer~(\Cref{fig:architecture}).
Following \citet{bhatt2024crossq}, we implement a joined forward pass to automatically calculate the \BN running statistics on the joined $(\vs,\va)$ and $(\vs',\va')$ distribution~\citep{bhatt2024crossq}, to successfully integrate \BN in the \RL loop.
In contrast to \citet{bhatt2024crossq}, we find that switching the order of normalization and \relu-activation leads to better performance. It has the added benefit that in this order, \BN's scale invariance is preserved for any activation function, as opposed to homogeneous ones only.
\XQC uses four hidden layers with $512$ neurons each.

\textbf{Cross-entropy Bellman loss.}
We use a \CFiftyOne-style categorical critic with $101$ atoms and a \CE loss~\citep{bellemare2017c51}.
We use standard reward normalization based on running statistics of the standard deviation of the return $R$ to effectively bound the $Q$ values to the support of our categorical critic
$
    \textstyle\hat{r}_t = r_t/\sigma(R)
$~\citep{engstrom2020implementation}.
Next to improving the loss landscape conditioning, another desirable property of the categorical \CE loss is to keep gradient norms bounded and thereby help keep the \ELR constant.

\textbf{Weight normalization.}
Enabled by \BN's scale-invariance property, we project the weights of each dense layer to the unit sphere after each gradient step. 
This normalization keeps the denominator of the \ELR constant, so it becomes practically constant when using the \CE loss (Figure \ref{fig:plasticity_metrics}), so \XQC maintains good plasticity.
With a constant \ELR, we can now leverage a learning rate schedule for Adam~\citep{kingma2014adam} as previously suggested \citep{lyle2024normalization,lee2025hyperspherical}.

\textbf{Vision encoder.}
For experiments on the \DMC vision-based environments, we use the standard \DRQ~\citep{yaratsmastering} image encoder. 
For a fair and direct comparison to \DRQ, we use its standard, unmodified vision encoder, which consists of convolutional layers alternated with \relu activations, followed by a linear layer, \LN, and a \tanh activation.
Our architectural modifications are confined to the subsequent \textsc{mlp} layers of the actor and critic.

\section{Experiments}\label{sec:experiemnts}
This section empirically validates our central hypothesis: that the synergistic combination of \BN, \WN, and a categorical \CE critic loss, designed to create a well-conditioned optimization landscape, directly translates into state-of-the-art sample efficiency and training stability.
We structure our experiments to first demonstrate \XQC's superior performance against strong baselines~(\Cref{sec:exp_sample_efficiency}),
then dissect the underlying mechanics through analysis of common plasticity metrics and the \ELR~(\Cref{sec:exp_plasticity}),
Finally, in \Cref{sec:exp_param_efficiency} we analyze 
computational efficiency, scaling properties,
and present a thorough ablation study~(\Cref{sec:ablations}) confirms the necessity of each of \XQC's architectural components.

\textbf{Evaluation metrics.}
For the main experiments we run 10 random seeds per environment for 1 million environment steps and for ablations 5 seeds, unless otherwise noted.
For statistically rigorous evaluations, we report the \IQM and 90\% \SBCI for all aggregate scores, following the recommended best practices of \citet{agarwal2021iqm}.
To aggregate \IQM return curves over multiple environments and benchmarks, each score needs to be normalized.
We follow standard practice, details in \Cref{app:aggregation_details}.
In aggregated bar charts, we present \AUC of the \IQM normalized return curve.
The \AUC captures both training speed as well as absolute performance simultaneously.
As such, it discriminates between two algorithms, which converge to the same performance but at different speeds, as opposed to the \IQM of the final performance.

\textbf{Benchmarks.}
To validate \XQC's effectiveness, we conduct comprehensive experiments across \NEnvsTotal continuous control tasks spanning five popular benchmark suites.
Our evaluation covers \NEnvsVision vision-based tasks from \DMC, plus and additional \NEnvsProprio proprioceptive tasks from \HB~\citep{sferrazza2024humanoidbench} (\NEnvsHB tasks), \DMC~\citep{tassa2018deepmindcontrolsuite} (\NEnvsDMC tasks),
\Myo~\citep{caggiano2022myosuite} (\NEnvsMYO tasks),
and \mujoco~\citep{todorov2012mujoco} (\NEnvsMUJOCO tasks). 
Our extensive evaluation shows \XQC's generality, using a single set of hyperparameters across all tasks. 

\textbf{Baselines.}
We compare to several strong, recent model-free 
baselines: 
\SIMBA~\citep{lee2025hyperspherical}, 
\BRO~\citep{nauman2024bigger}, 
\MrQ~\citep{fujimoto2025mrq}, 
\BRC~\citep{nauman2025brc}, 
\XQWN~\citep{palenicek2025scaling},
\SAC~\citep{haarnoja2018sac},
When available, we use the respective authors' evaluation results; otherwise, we run experiments using their official open-source implementations. Full details on baseline results are provided in \Cref{app:baselines}.

\begin{figure}[t!]
    \centering
    \includegraphics[width=\linewidth]{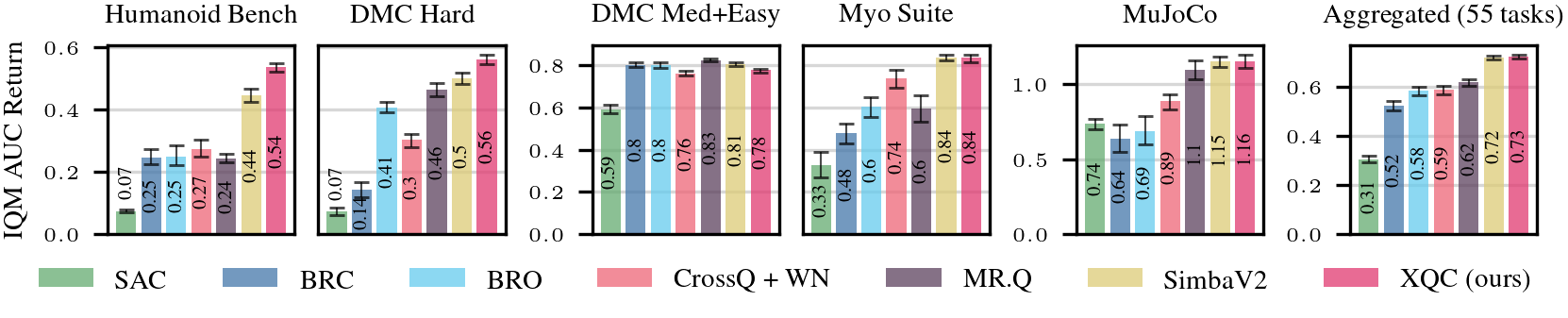}
    \vspace{-2.2em}
    \caption{
    \textbf{\XQC achieves state-of-the-art sample efficiency across \NEnvsProprio proprioceptive continuous control tasks.}
    We report the \IQM \AUC of normalized returns.
    Error bars denote 90\% \SBCI{}s.
    The right column shows total aggregated performance across the benchmarks (\NEnvsProprio tasks).
    \XQC matches or outperforms strong baselines, especially on the hardest \DMC and \HB tasks,
    while using a simpler and smaller architecture (see \Cref{sec:exp_param_efficiency}).
    }
    \label{fig:sample_efficiency_proprio}
\end{figure}

\subsection{Sample efficiency results}\label{sec:exp_sample_efficiency}
We start our experiments by investigating the training performance in terms of sample efficiency and comparing it to state-of-the-art baselines. 
All of these results use 10 seeds per environment.
First, we present the proprioception-based results and then the vision-based tasks.

\textbf{Reinforcement learning from proprioception.}
\Cref{fig:sample_efficiency_proprio} shows that \XQC matches our outperforms strong baselines \SIMBA, \MrQ, \BRO and \XQWN on all 4 benchmarks.
The rightmost column shows that on average \XQC performs as well \SIMBA while using significantly less network parameters and a substantially simpler architecture~\Cref{fig:architecture}.
Notably \XQC shows exceptional performance on the most complex tasks \HB and \DMC-\texttt{hard}.
These enm induce notoriously difficult and ill-conditioned optimization landscapes. \XQC's superior performance and learning speed suggest that its well-conditioned critic---characterized by a stable \ELR and bounded gradients as shown in \Cref{sec:exp_plasticity}---is fundamentally better equipped to handle the non-stationary targets and bootstrapping errors inherent in these challenging domains.

\begin{figure}[t!]
    \centering
    \includegraphics[width=\linewidth]{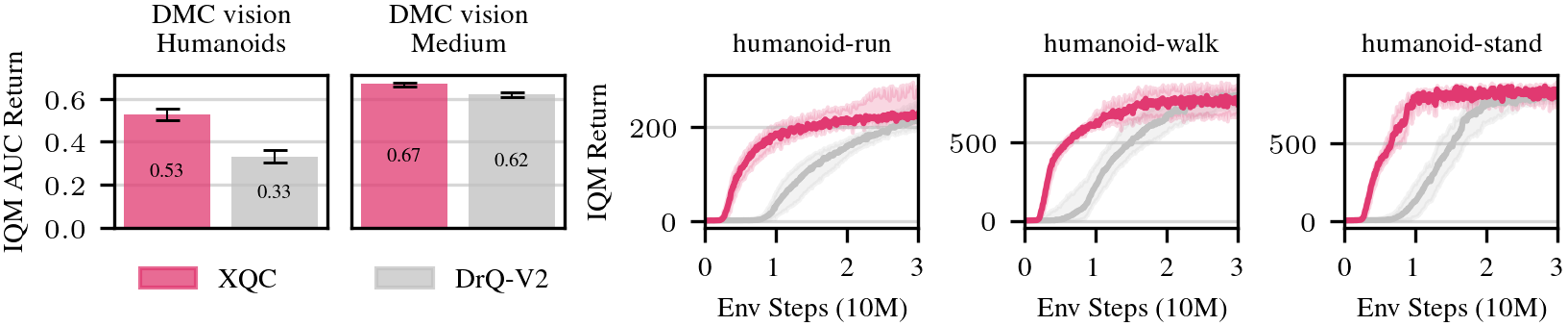}
    \vspace{-2em}
    \caption{
    \textbf{\XQC improves sample efficiency on \NEnvsVision vision-based \DMC tasks.}
    The left two columns show aggregated \IQM \AUC, demonstrating a significant performance advantage over the strong \DRQ baseline, particularly on the difficult \humanoid tasks.
    The right three columns show full training curves for the three \humanoid environments (\IQM over 10 seeds), highlighting \XQC's significantly better sample-efficiency.
    For these experiments, \XQC uses the standard \DRQ encoder and hyperparameters, isolating performance gains to our proposed well-conditioned critic architecture.
    }
    \label{fig:sample_efficiency_vision}
\end{figure}

\textbf{Reinforcement learning from pixels.}
On the vision-based \DMC environments, we compare \XQC to \DRQ. For these results, we re-implemented \XQC in the official \DRQ codebase. We used the \DRQ encoder and the same hyperparameters as the original \DRQ to make the comparison as fair as possible.
\Cref{fig:sample_efficiency_vision} shows that on most tasks, \XQC outperforms or at least matches \DRQ performance.
This is most pronounced on the much more challenging \texttt{humanoid} tasks.
Learning the vision-encoder from scratch requires a large number of samples in itself. We hypothesize that this is why the performance increase of \XQC is smaller on the easier tasks and much more pronounced on the humanoids, which have a $10\times$ overall runtime.

\subsection{Plasticity analysis}\label{sec:exp_plasticity}
Analysing the improvement of common plasticity metric confirms the effectiveness of \XQC{}s design principles. \Cref{fig:plasticity_metrics} presents plasticity metrics aggregated over all \NEnvsProprio proprioceptive tasks.
\XQC w/o \WN's growing parameter norms decrease the \ELR towards zero over time, reconfirming the findings of~\citet{lyle2024normalization} and~\citet{palenicek2025scaling}.
We notice that the \ELR appears directly coupled to the \textit{gradient norm}, for all architectures employing \WN.
While \XQC \MSE controls the parameter norm, its gradients are unbounded and heavily influenced by outliers; consequently, its gradient norm and \ELR grow over the course of training by about one order of magnitude (requiring a second \texttt{y-axis} to compare).
\XQC's \CE critic loss removes this disturbance, directly reflected in remarkably stable gradient norm and \ELR throughout the course of training, which are many orders of magnitude smaller.
We show per benchmark plasticity metrics in~\Cref{app:plasticity}.

\begin{figure}[t!]
    \centering
    \includegraphics[width=\linewidth]{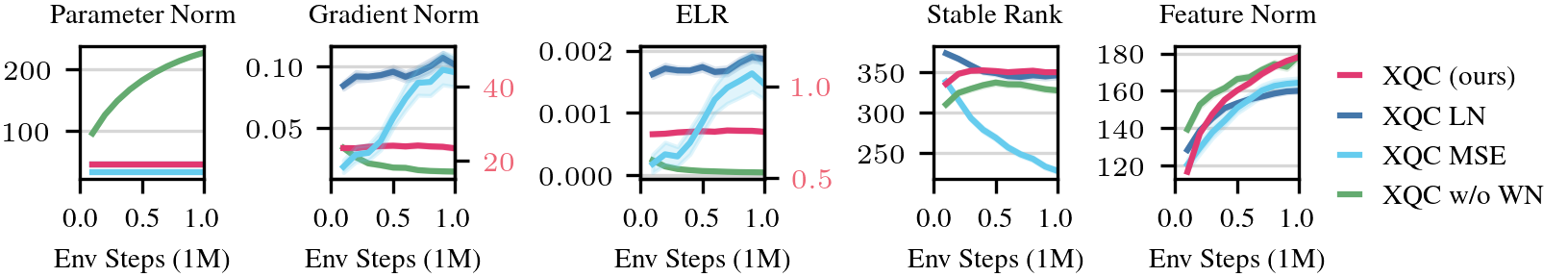}
    \caption{
    \textbf{\XQC's architecture creates exceptionally stable learning dynamics.}
    For \XQC, \BN+\WN stabilizes the parameter norm, while \BN+\CE keep the gradient norm and \ELR near constant.
    }
    \label{fig:plasticity_metrics}
\end{figure}

\subsection{Parameter and compute effciency}\label{sec:exp_param_efficiency}
\begin{figure}[b!]
    \centering
    \includegraphics[width=\linewidth]{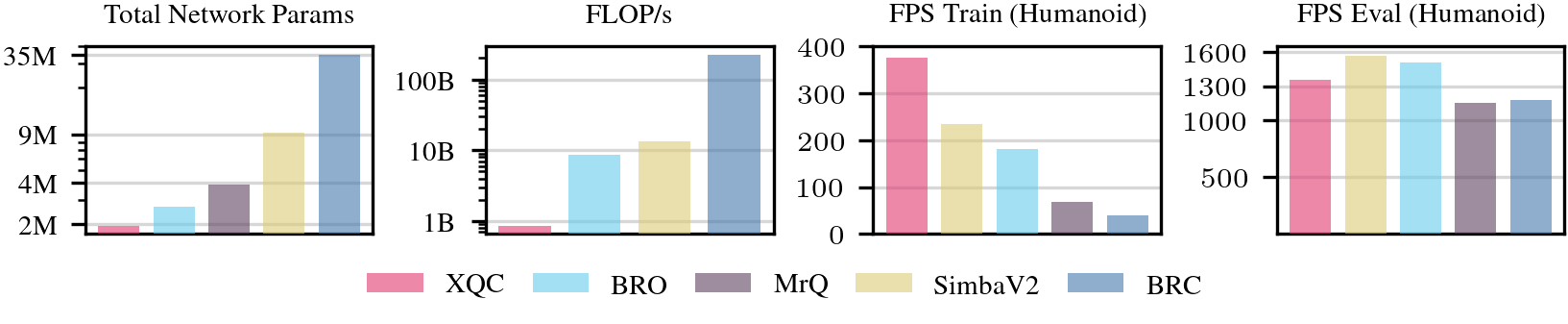}
    \vspace{-2em}
    \caption{
    \textbf{\XQC is significantly more parameter and compute efficient than
    competitive baselines.}
    $\sim4.5\times$ fewer parameters and $\sim5\times$ fewer \FLOPS than \SIMBA and \BRO and $>100\times$ fewer \FLOPS than \BRC.
    The computational efficiency translates into a significantly higher training \textsc{fps} measured on the \mujoco \humanoid environment on an \textsc{nvidia rtx} 4090 workstation.
    }
    \label{fig:efficiency}
\end{figure}
\XQC achieves competitive sample efficiency while requiring $\sim4.5\times$ fewer parameters than \SIMBA~(\Cref{fig:efficiency}).
This parameter efficiency directly results in high computational efficiency with $\sim5\times$ fewer \FLOPS than \SIMBA and \BRO and $>100\times$ fewer \FLOPS than \BRC.
We conjecture that \XQC's superior computational efficiency is rooted in its well-conditioned architecture.
With respect to practical performance, \XQC shows $60\%$ higher \frames during training, measured on the \mujoco \humanoid on a \textsc{nvidia rtx} 4090 workstation, resulting in significantly faster wall-clock times.
\XQC's evaluation performance is measured at $1355$\frames, slightly lower the \SIMBA baseline due to its slightly larger policy networks, but remains more than fast enough typical control frequencies seen in deep reinforcement learning (10--50 Hz).

\subsection{XQC scaling behaviour and architecture ablations}\label{sec:ablations}
\begin{figure}[t!]
    \centering
    \includegraphics[width=\linewidth]{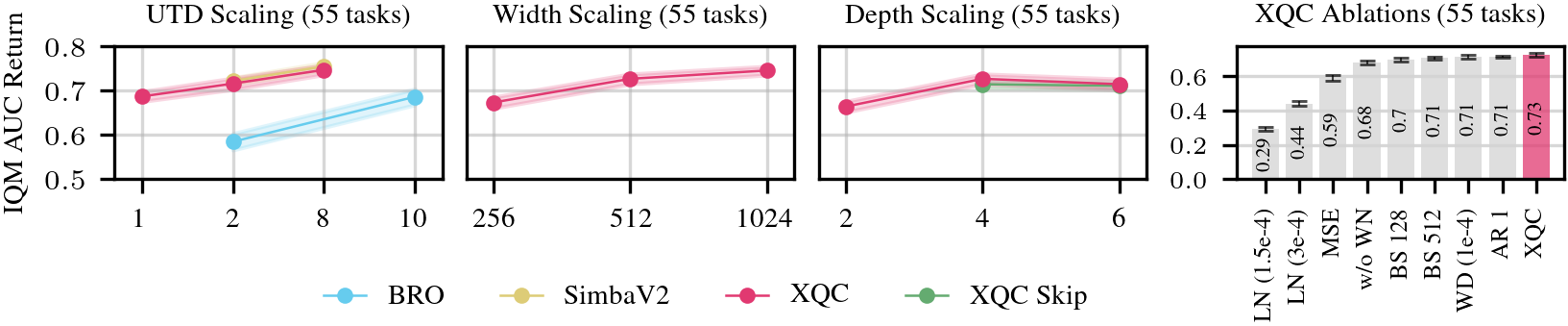}
    \caption{
    \textbf{\XQC scales stably with \UTD, network depths, and widths.}
    Columns 1--3:
    \XQC's scaling in terms of \UTD, layer width, and layer depth.
    Increasing compute and model capacity generally improves or maintains performance, demonstrating that our well-conditioned design is robust to scaling.
    Results are aggregated across all \NEnvsProprio proprioception tasks and 5 seeds each.
    Right column:
    We perform multiple architectural and parameter ablations with respect to the full \XQC algorithm.
    \XQC is robust to changing the action repeat (\textsc{ar}=1),
    introducing weight decay (\textsc{wd}),
    and different batch sizes (\textsc{bs}$\in\{128,512\}$).
    We further ablate all three of \XQC's main components,
    first without \WN (w/o \WN),
    second replacing the \CE loss with an \MSE loss (\MSE),
    and third replacing \BN with \LN.
    Each component's removal results in a significant performance drop, showing their synergistic contribution.
    }
    \label{fig:scaling}
\end{figure}

\Cref{fig:scaling} colums 1--3 demonstrate \XQC's scaling ability.
\XQC improves performance with increasing \UTD at a similar slope to \SIMBA, with a fraction of the parameters (\Cref{sec:exp_param_efficiency}).
Similarly, \XQC improves or maintains performance with larger and deeper networks.
These results demonstrate robustness towards different hyperparameters and \XQC's ability to scale stably, enabled by its well-conditioned architectural design.
This property is desirable since increasing compute always represents a trade-off between sample efficiency and wall-clock or energy. It allows practitioners to use their available budget most efficiently.
For depth scaling, we also ablate the influence of skip connections but find no significant effect for the investigated depths.
Additionally, we investigate the influence of different batch sizes $\in\{128, 512\}$ compared to the default of $256$, the use of weight decay, and the impact of action repetition.
Neither of these ablations shows a significant influence on \XQC.
\Cref{sec:additional_conditioning} presents additional conditioning experiments for the tested components.
To demonstrate the synergy between the proposed architectural components, we ablate each of the three main components of the proposed \XQC architecture: \BN, \WN, and the \CE distributional critic.
Results for all \NEnvsProprio and 5 seeds are shown in \Cref{fig:scaling} with per benchmark ablations in \Cref{fig:ablations}.
Our analysis confirms that each of the components is vital, especially in the most difficult \DMC-\texttt{hard} and \HB suits.
The largest drop in performance occurs when switching \BN for \LN.
We further investigate \LN with half the learning rate, accounting for the increased \ELR of \LN~(See~\Cref{fig:plasticity_metrics}); however, this leads to even worse performance. This is in line with our finding that \LN-based architectures show lower performance due to worse conditioning.
While using a \MSE critic loss has the second most significant influence, the removal of \WN is still substantial, but shows the lowest overall impact.
In summary, combining each of the three components is vital for \XQC's performance.

\section{Related Work}
\label{sec:related_work}
Our work is positioned at the intersection of several key research areas in deep \RL that all attempt to improve sample efficiency. 
A prevailing trend for improving sample efficiency has been scaling along different axes.
\textit{Compute scaling}, particularly increasing the \UTD ratio, has recently been a major focus.
However, naively increasing \UTD can lead to instability and overfitting on early experience~\citep{nikishin2022primacy}.
Researchers have suggested many different regularizers to stably increase the \UTD; From full parameter-resets~\citep{nikishin2022primacy,doro2022replaybarrier},
to critic ensembles~\citep{chen2021redq} and drop-out~\citep{hiraoka2021droq},
to normalization layers like \LN~\citep{hiraoka2021droq,lyle2024normalization,nauman2024bigger,lee2024simba}, \BN~\citep{bhatt2024crossq}, hyper-spherical normalization~\citep{hussing2024dissecting,lee2025hyperspherical} and spectral normalization~\citep{bjorck2022spectralnorm}.
Recently, works have found that the combination of normalization layers together with \WN can stabilize the \ELR, helping against loss of plasticity~\citep{lyle2024normalization} and also enable scaling \RL to high \UTD ratios~\citep{palenicek2025scaling}.
Recent works have worked on developing scaling laws for \RL~\citep{rybkin2025valueScalesPredictably,fu2025compute}.
\textit{Network scaling}, is another path authors are exploring to increase sample efficiency.
\citet{bhatt2024crossq} showed that \BN allowed them to significantly scale the layer width.
Since then, authors have looked into specific \LN-based architectures~\citep{nauman2024bigger,lee2024simba} and network sparsity~\citep {man2025etworkSparsity}.
\citet{lee2024simba} propose a `simplicity bias' score, computed using an FFT and scoring 'simplicity' higher for functions with lower frequency content across random initializations. This score has no theoretical justification relating it to sample efficiency or generalization.
Another line of research attempts to scale network sizes dynamically during training~\citep{liu2025neuroplastic,kang2025forget}.
Concurrent work \citet{castanyer2025stable} combine second-order optimization and multi-skip residual connections to improve scaling and monitor the trace of the Hessian for deep value-based \RL. 
Orthogonal to our work, recent methods revisit the target network design itself. \citet{vincent2025bridging} share features between online and target networks to reduce memory while preserving stability, and \citet{hendawy2025minto} take the minimum of both estimates for bootstrapping to accelerate learning.
\textit{Model-based} methods scale computation by learning a separate dynamics model, which is leveraged in the \RL loop~\citep{janner2019mbpo,hafner2020dreamer,palenicek2023diminishing,hansen2024tdmpc2}.

\section{Conclusion \& Future Work}
In this work, we shifted the focus from the prevailing pure scaling goal in deep \RL and instead focus on improving the critic's optimization landscape.
Through an eigenvalue analysis of the critic's Hessian, we demonstrate that specific architectural choices, namely batch normalization, weight normalization, and a distributional cross-entropy loss, create a better optimization landscape with a condition number orders of magnitude smaller during learning.
This superior conditioning translates directly into learning performance gains.
We propose \XQC, an algorithm embodying these principles, which achieves state-of-the-art sample efficiency across \NEnvsTotal continuous control tasks from proprioception and vision domains.
\XQC accomplishes this performance with significantly fewer parameters than competing methods, underscoring that a principled focus on optimization fundamentals can yield greater performance and efficiency than brute-force scaling alone.
For future work, insights from \XQC can be used to accelerate actor-critic model-based and imitation learning algorithms, e.g., \citet{janner2019mbpo}, \citet{watson2023coherent}.
\XQC performance could be further accelerated by incorporating prior datasets, e.g., \citet{ball2023efficient}.

\section*{Acknowledgements}
This research was funded by the research cluster ``Third Wave of AI'', funded by the excellence program of the Hessian Ministry of Higher Education, Science, Research and the Arts, hessian.AI and by the Deutsche Forschungsgemeinschaft (DFG, German Research Foundation) under Germany's Excellence Strategy (EXC-3057/1 ``Reasonable Artificial Intelligence'', Project No. 533677015).
It was further supported by a UKRI/EPSRC Programme Grant [EP/V000748/1] and partially supported by the German Federal Ministry of Research, Technology and Space (BMFTR) under the Robotics Institute Germany (RIG).

\section*{Reproducibility statement}
We took special care to ensure this work is reproducible and will make the code open source upon acceptance. 
To ease reproducibility, algorithm details are explained \Cref{sec:xqc}, 
all hyperparameters are listed in~\Cref{sec:hyperparameters},
and training curves are shown~\Cref{app:all_curves_vision,app:all_curves_proprio}.

\section*{Large language model usage}
A large language model was helpful in polishing writing, improving reading flow, and identifying
remaining typos.

\bibliography{iclr2026_conference}
\bibliographystyle{iclr2026_conference}

\appendix

\section{Hyperparameters}\label{sec:hyperparameters}
\Cref{tab:hyperparameter_our_experiments} summarizes all proprioception-based experiments' hyperparameters.
\[
\gamma = \text{clip} \left( \frac{\frac{T}{5}-1}{\frac{T}{5}}, [0.95, 0.995] \right),
\]
where \(T\) denotes the effective episode length, calculated by dividing the episode length by the number of repeated actions.
We use the default hyperparameters used in their respective GitHub repositories for all other baselines.
One exception is \BRC, where we reduced the number of parameters from the default 256M to 64M. As noted by the authors, using more than 64M parameters does not provide additional benefit in the single-task setting, which is the focus of our work.

\Cref{tab:hyper_drq} contains all hyperparameters for the vision-based experiments based on the official \DRQ codebase.

\begin{sidewaystable}[]
\caption{Hyperparameters for \XQC and baselines on all proprioception tasks.}\label{tab:hyperparameters}
\footnotesize
\begin{tabular}{llllllll}
\hline
\textbf{Hyperparameter}         & \XQC                    & \XQWN            & \SAC                    & \SIMBA                                                                  & \BRO                                                          & \BRC                    & \MrQ                  \\ \hline
Block design & \makecell[l]{Dense(dim) \\ BN \\ ReLU} & \makecell[l]{Dense(dim) \\ ReLU \\ BN} & \makecell[l]{Dense(dim) \\ ReLU}  & \makecell[l]{Dense(4 $\times$ dim) \\ Scaler \\ ReLU \\ Dense(dim) \\ L2 Norm \\ LERP \\ L2 Norm}  & \makecell[l]{Dense(dim) \\ LN \\ ReLU \\ Dense \\ LN  \\ Skip Connection} & \makecell[l]{Dense(dim) \\ LN \\ ReLU \\ Dense \\ LN  \\ Skip Connection} & \makecell[l]{Dense(dim) \\ LN \\ Activation}  \\ \\
Critic learning rate            & 0.0003                 & 0.0003                 & 0.0003                 & 0.0001                                                                    & 0.0003                                                       & 0.0003                 & 0.0003                \\
Critic hidden dim               & 512                    & 512                    & 256                    & 512   & 512                                                          & 2048                   & 512                   \\
Critic number of blocks         & 4                      & 2                      & 2                      & 2                                                                         & 2                                                            & 2                      & 3                     \\
Actor learning rate             & 0.0003                 & 0.0003                 & 0.0003                 & 0.0001                                                                    & 0.0003                                                       & 0.0003                 & 0.0003                \\
Actor number blocks             & 256                    & 256                    & 256                    & 128    & 256                                                          & 256                    & 512                   \\
Actor number of blocks          & 4                      & 2                      & 2                      & 1                                                                         & 1                                                            & 1                      & 3                     \\
Policy update delay             & 3                      & 3                      & 1                      & 1                                                                         & 1                                                            & 1                      & 1                     \\
Initial temperature             & 0.01                   & 0.01                   & 1                      & 0.01                                                                      & 1.0                                                          & 0.1                    & -                     \\
Temperature learning rate       & 0.0003                 & 0.0003                 & 0.0003                 & 0.0001                                                                    & 0.0003                                                       & 0.0003                 & -                     \\
Target entropy                  & $\left| A \right| / 2$ & $\left| A \right| / 2$ & $\left| A \right| / 2$ & $\left| A \right| / 2$                                                    & $\left| A \right| / 2$                                       & $\left| A \right| / 2$ & -                     \\
Target network momentum         & 0.005                  & 0.005                  & 0.005                  & 0.005                                                                     & 0.005                                                        & 0.005                  & 1.0                   \\
Target network update frequency & 1                      & 1                      & 1                      & 1                                                                         & 1                                                            & 1                      & 250                   \\
Number of critics               & 2                      & 2                      & 2                      & \begin{tabular}[t]{@{}l@{}}2 on MuJoCo / HB\\ 1 on DMC / Myo\end{tabular} & 2                                                            & 2                      & 2                     \\
Discount                        & Heuristic              & Heuristic              & Heuristic              & Heuristic                                                                 & 0.99                                                         & 0.99                   & 0.99                  \\
Optimizer                       & Adam                   & AdamW                  & Adam                   & Adam                                                                      & AdamW                                                        & AdamW                  & AdamW                 \\
Weight Decay                    & -                    & 0.01                   & -                    & -                                                                       & 0.0001                                                       & 0.0001                 & 0.0001                \\
Categorical Support             & [-5, 5]                & -                      & -                      & [-5, 5]                                                                   & -                                                            & [-10, 10]              & -             \\
Critic loss                 & CE                    & MSE                     & MSE                      & CE                                                                       & MSE                                                            & CE                    & L1 Loss                    \\
UTD                             & 2                      & 2                      & 1                      & 2                                                                         & \begin{tabular}[t]{@{}l@{}}2 BRO Small\\ 10 BRO\end{tabular} & 2                      & 1                     \\
Batch size                      & 256                    & 256                    & 256                    & 256                                                                       & 128                                                          & 1024                   & 256                   \\ \\
Action repeat & \makecell[l]{MuJoCo: 1 \\ DMC: 2 \\ HB: 2 \\ Myo: 2} & \makecell[l]{MuJoCo: 1 \\ DMC: 2 \\ HB: 2 \\ Myo: 2} & \makecell[l]{MuJoCo: 1 \\ DMC: 2 \\ HB: 2 \\ Myo: 2} & \makecell[l]{MuJoCo: 1 \\ DMC: 2 \\ HB: 2 \\ Myo: 2} & \makecell[l]{MuJoCo: 1 \\ DMC: 1 \\ HB: 1 \\ Myo: 1} & \makecell[l]{MuJoCo: 1 \\ DMC: 2 \\ HB: 2 \\ Myo: 2} & \makecell[l]{MuJoCo: 1 \\ DMC: 2 \\ HB: 2 \\ Myo: 2} \\
\hline
\end{tabular}
\label{tab:hyperparameter_our_experiments}
\end{sidewaystable}

\begin{table}[]
\caption{Hyperparameters for vision-based \RL tasks.}\label{tab:hyper_drq}
\footnotesize
\begin{tabular}{lll}
\hline
\textbf{Hyperparameter}         & \XQC & \DRQ    \\ \hline
Block design                    &  \makecell[l]{Dense(dim) \\ BN \\ ReLU}            &  \makecell[l]{Dense(dim) \\ ReLU}    \\  \\
Critic learning rate            & 0.0001       & 0.0001 \\
Critic hidden dim               & 1024          & 1024   \\
Critic number of blocks         & 4            & 2      \\
Actor learning rate             & 0.0001       & 0.0001 \\
Actor hidden dim                & 1024          & 1024 \\
Actor number of blocks          & 4            & 3      \\
Target network momentum         & 0.01         & 0.01   \\
Target network update frequency & 1            & 1      \\
Number of critics               & 2            & 2      \\
Discount                        & 0.99         & 0.99   \\
Optimizer                       & Adam         & Adam   \\
Categorical Support             & {[}-5, 5{]}  & -      \\
Critic Loss                     & CE           & MSE    \\
UTD                             & 0.5          & 0.5    \\
Batch size                      & 256          & 256    \\
Replay buffer capacity          & 1M          & 1M    \\
N-step returns                  & 3            & 3      \\
Feature dim                     & 50           & 50     \\
Exploration stddev. clip        & 0.3          & 0.3    \\
Action Repeat                   & 2            & 2      \\ 
\begin{tabular}[c]{@{}l@{}}Exploration stddev. schedule\\ (Defined by task difficulty)\end{tabular}     &  \makecell[l]{easy: linear(1.0, 0.1, 100K) \\ medium: linear(1.0, 0.1, 500K) \\ hard: linear(1.0, 0.1, 2M)} & \makecell[l]{easy: linear(1.0, 0.1, 100K) \\ medium: linear(1.0, 0.1, 500K) \\ hard: linear(1.0, 0.1, 2M)}        \\ \hline
\end{tabular}
\end{table}

\newpage
\section{Environment aggregation details}\label{app:aggregation_details}
As the magnitude of returns varies across environments, we normalize them for comparability before aggregating~\citep{agarwal2021iqm}. We normalize scores to be between 0 and 1.
Where the normalization protocols are benchmark-specific and follow standard practice.

\textbf{For MuJoCo and Humanoidbench,} we compute the normalized score as 
$$\hat{x} = \frac{x - \text{Random Score}}{\text{Target Score} - \text{Random Score}},$$
where the random scores are obtained using a uniformly random policy~\citep{fu2020d4rl}. Target scores are taken from a trained TD3 policy in \mujoco, and are provided by the authors for \HB, where they represent the threshold required to mark a task as solved. 

\textbf{For \DMC tasks,} we normalize by dividing the final score by 1000, the maximum achievable return. 

\textbf{MyoSuite} tasks require no normalization, as performance is already expressed in percentage-based success rates.

\newpage
\section{Baselines}\label{app:baselines}
In this section, we briefly describe how the results were collected for every baseline we present in this work. Additionally, all hyperparameters are listed in \Cref{sec:hyperparameters}.

\paragraph{\SIMBA~\citep{lee2025hyperspherical}.} We used the results made publicly available on the official GitHub repository. The results are based on 10 seeds. For \SIMBA (small) we ran the code from the official codebase ourselves, for 10 seeds.

\paragraph{\XQWN~\citep{palenicek2025scaling}.} We ran all experiments ourselves using our codebase for 10 seeds.

\paragraph{\BRO~\citep{nauman2024bigger}.}
We used the publicly available results on the official \SIMBA GitHub repository. We only considered the `small' version of BRO, which uses a UTD ratio of 2. The results are based on 5 seeds.

\paragraph{\MrQ~\citep{fujimoto2025mrq}.}
For \mujoco and \DMC, we used the results provided by \SIMBA on their official GitHub repository, which are based on 10 seeds.  We conducted experiments for \Myo and \HB ourselves, by running the official \MrQ codebase using 5 random seeds due to computational reasons. We matched the action repeat used in our experiments to ensure a fair comparison.

\paragraph{\SAC~\citep{haarnoja2018sac}.} We ran all experiments ourselves using our codebase. We used the default \SAC hyperparameters and ran 5 seeds for every environment.

\paragraph{\BRC~\citep{nauman2025brc}.} Using the official \BRC GitHub codebase, we experiment ourselves for 3 seeds, due to computational constraints.
While we used their default settings as reported in the paper, reducing the parameters to 64M, since \citet{nauman2025brc} noted that using more than 64M parameters provides no benefit for single-task settings.
Additionally, we used the same action repeat environment wrapper used in all our other experiments, ensuring a fair comparison.

\paragraph{\DRQ~\citep{hiraoka2021droq}.}
We used the results reported in the offical GitHub repository based on 10 seeds.
Our vision-based \XQC experiments are based on the \DRQ codebased for a fair comparison.

\newpage
\section{All Training Curves: Reinforcement Learning From Vision-based DMC Environments}\label{app:all_curves_vision}
Results from \RL on the vision-based DMC benchmarks. We compare to \DRQ~\citep{yaratsmastering}.

\begin{figure}[h!]
    \centering
    \includegraphics[width=\linewidth]{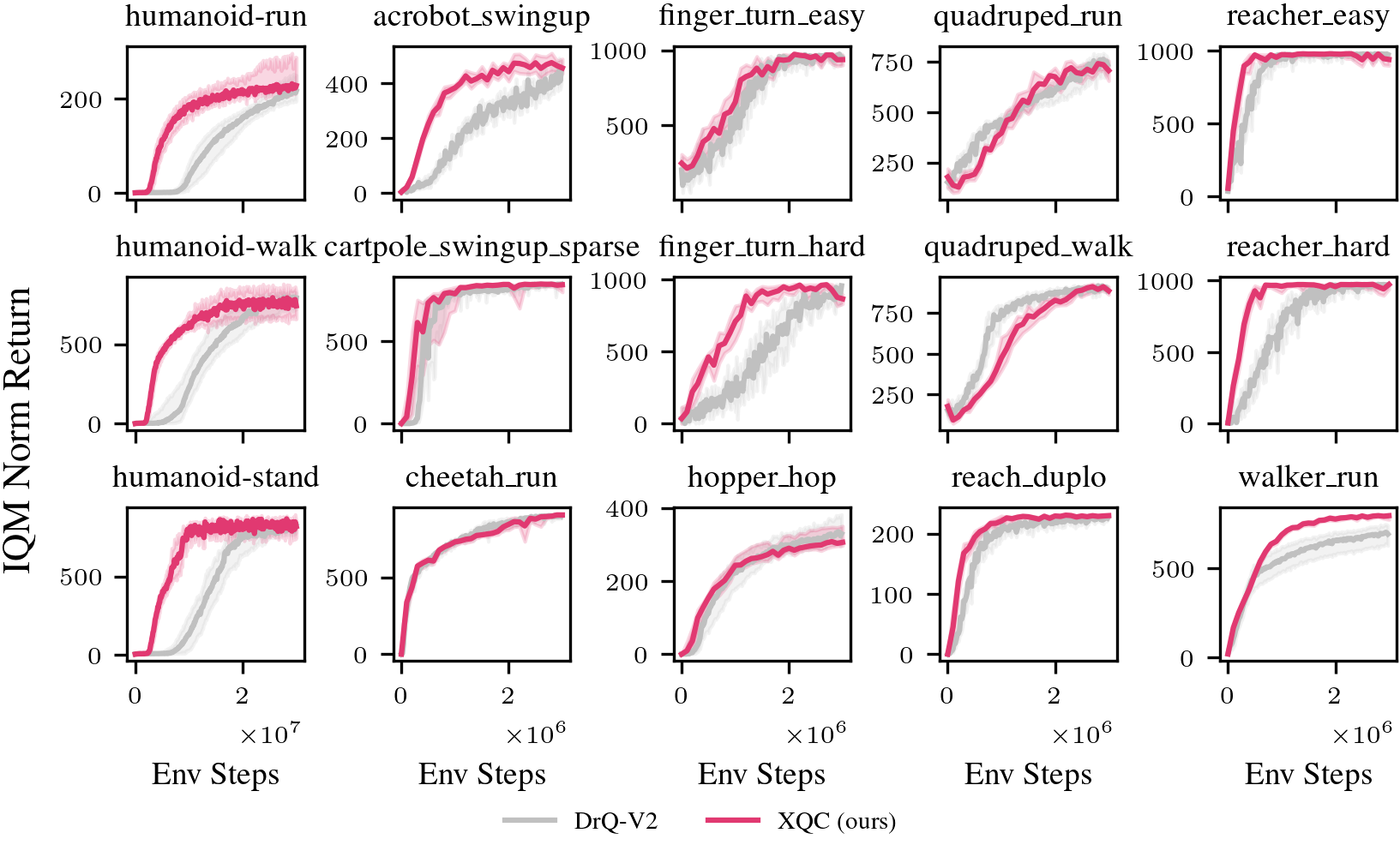}
    \caption{\XQC and \DRQ training curves for each of the \NEnvsVision vision-based \DMC tasks. We show the \IQM and 90\% \SBCI aggregated over 10 seeds per environment.}
\end{figure}

\begin{figure}[h!]
    \centering
    \includegraphics[width=\linewidth]{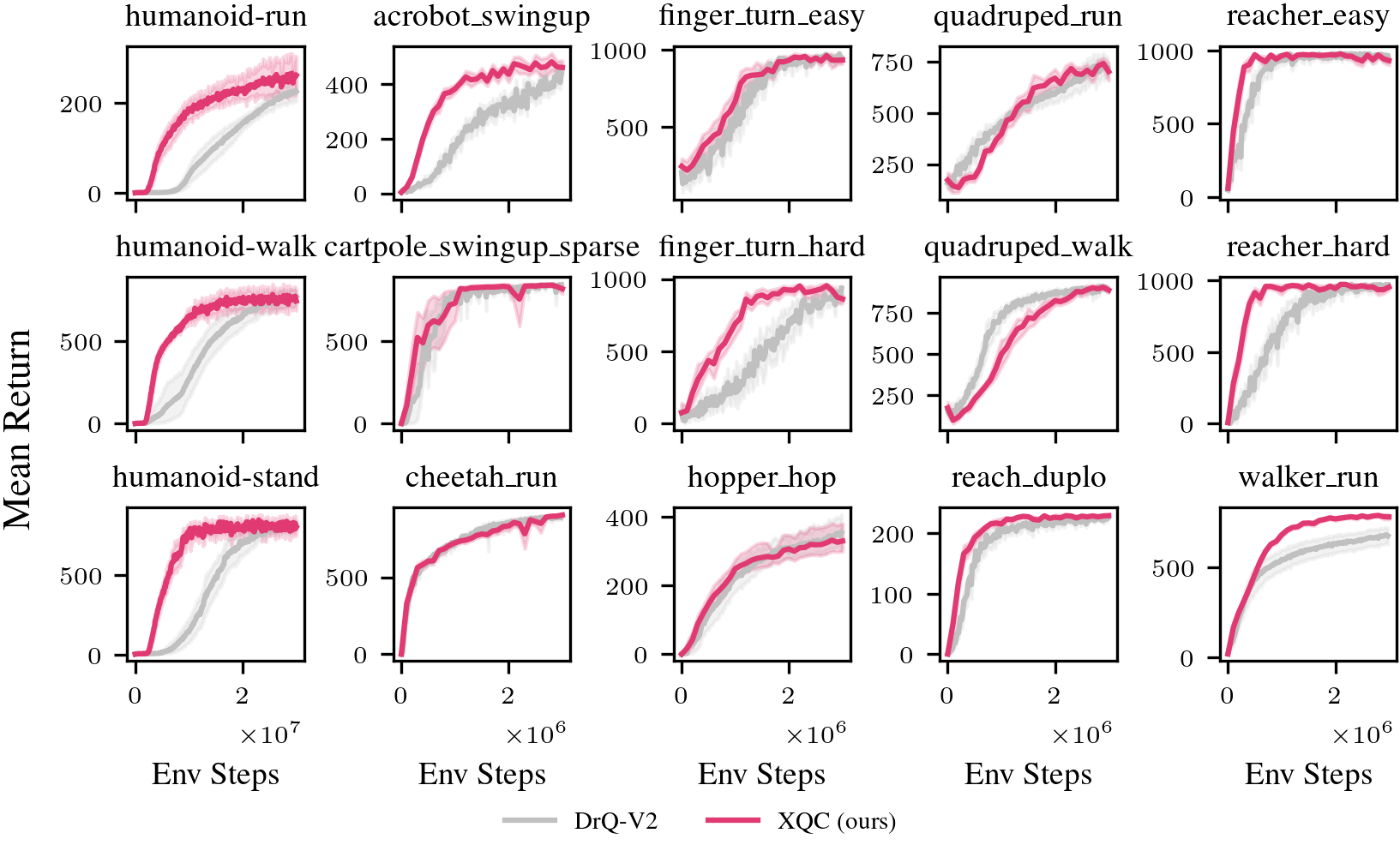}
    \caption{\XQC and \DRQ training curves for each of the \NEnvsVision vision-based \DMC tasks. We show the mean and 90\% \SBCI aggregated over 10 seeds per environment.}
\end{figure}

\newpage
\section{All Training Curves: Reinforcement Learning From Proprioception}\label{app:all_curves_proprio}
All results for the proprioception continuous control benchmarking tasks with \IQM and mean.

\subsection{HumanoidBench}
\begin{figure}[h!]
    \centering
    \includegraphics[width=\linewidth]{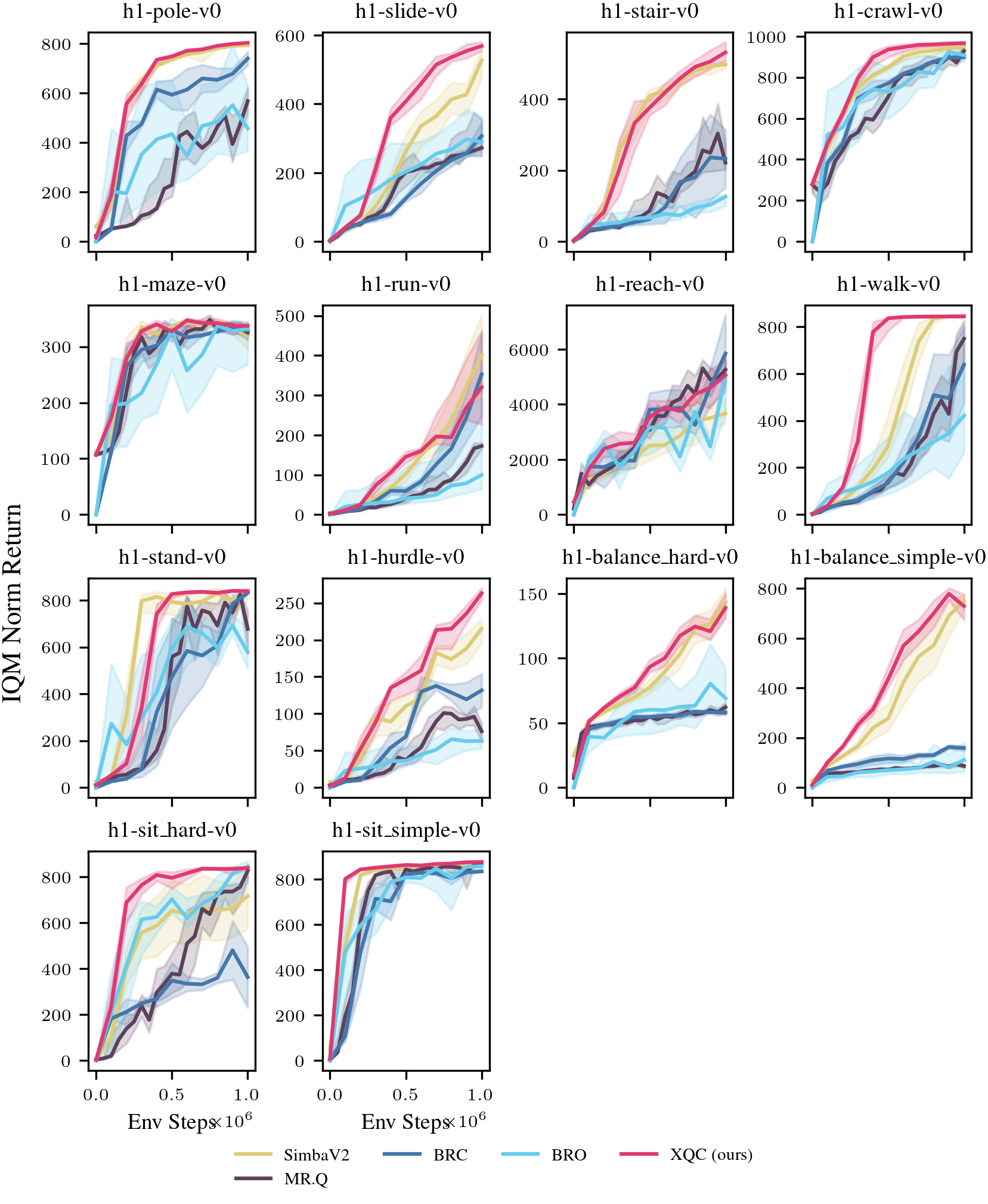}
    \caption{\XQC and baseline training curves for each of the \NEnvsHB \HB tasks. We show the \IQM and 90\% \SBCI aggregated per environment.}
\end{figure}

\begin{figure}[h!]
    \centering
    \includegraphics[width=\linewidth]{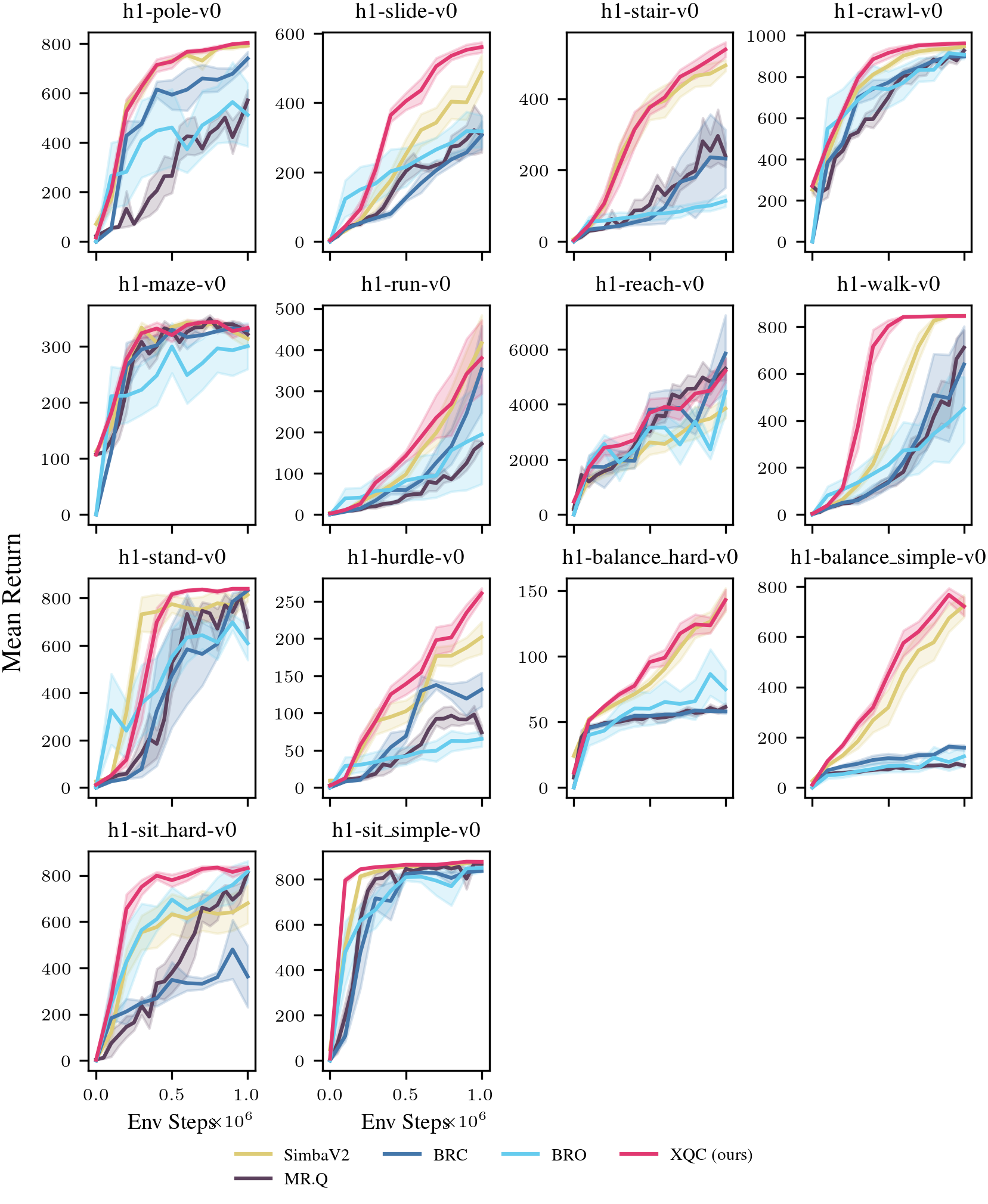}
    \caption{\XQC and baseline training curves for each of the \NEnvsHB \HB tasks. We show the mean and 90\% \SBCI aggregated per environment.}
\end{figure}

\clearpage
\subsection{DeepMind Control Suite}
\begin{figure}[h!]
    \centering
    \includegraphics[width=\linewidth]{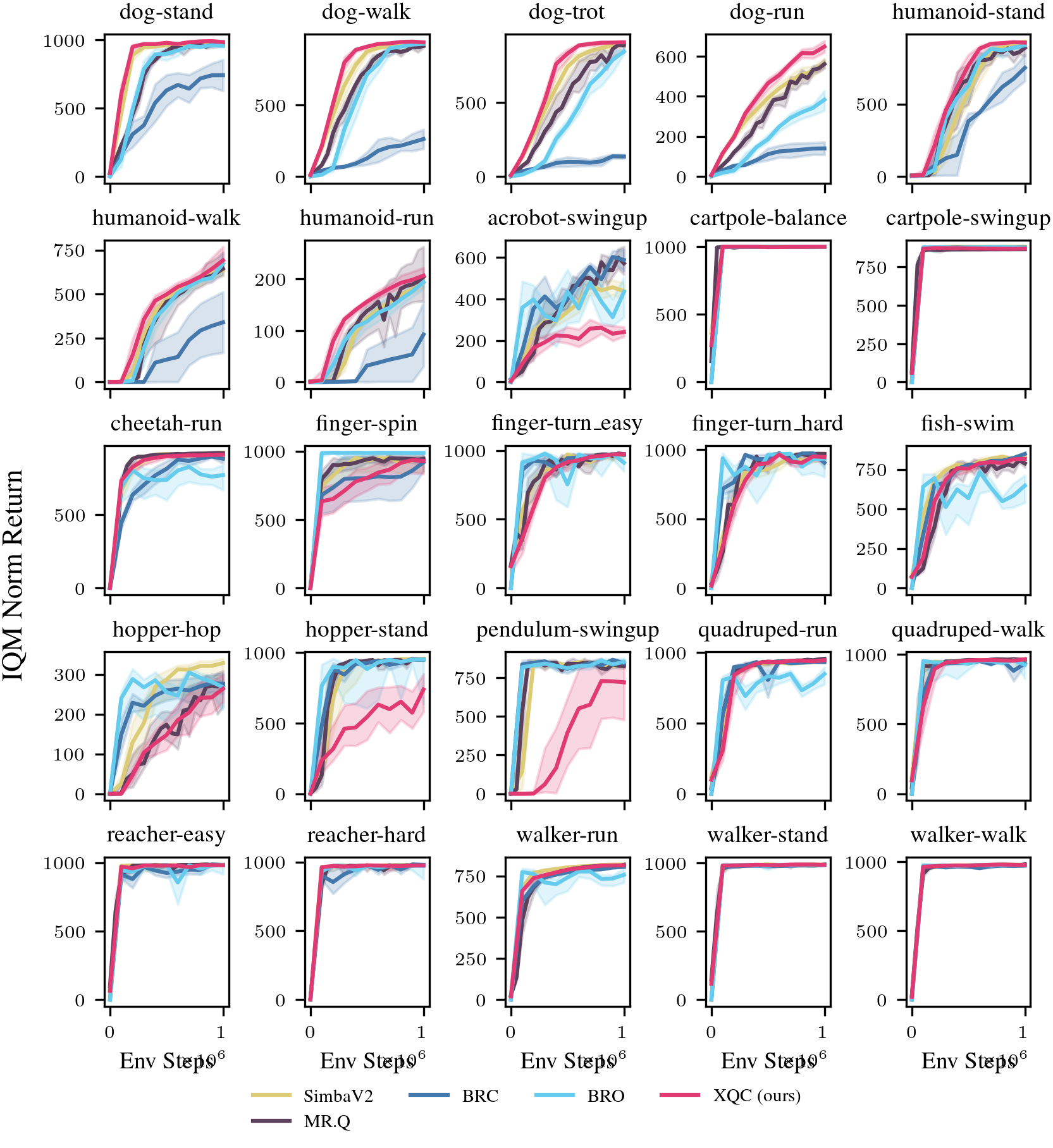}
    \caption{\XQC and baseline training curves for each of the \NEnvsDMC \DMC tasks. We show the \IQM and 90\% \SBCI aggregated per environment.}
\end{figure}

\begin{figure}[h!]
    \centering
    \includegraphics[width=\linewidth]{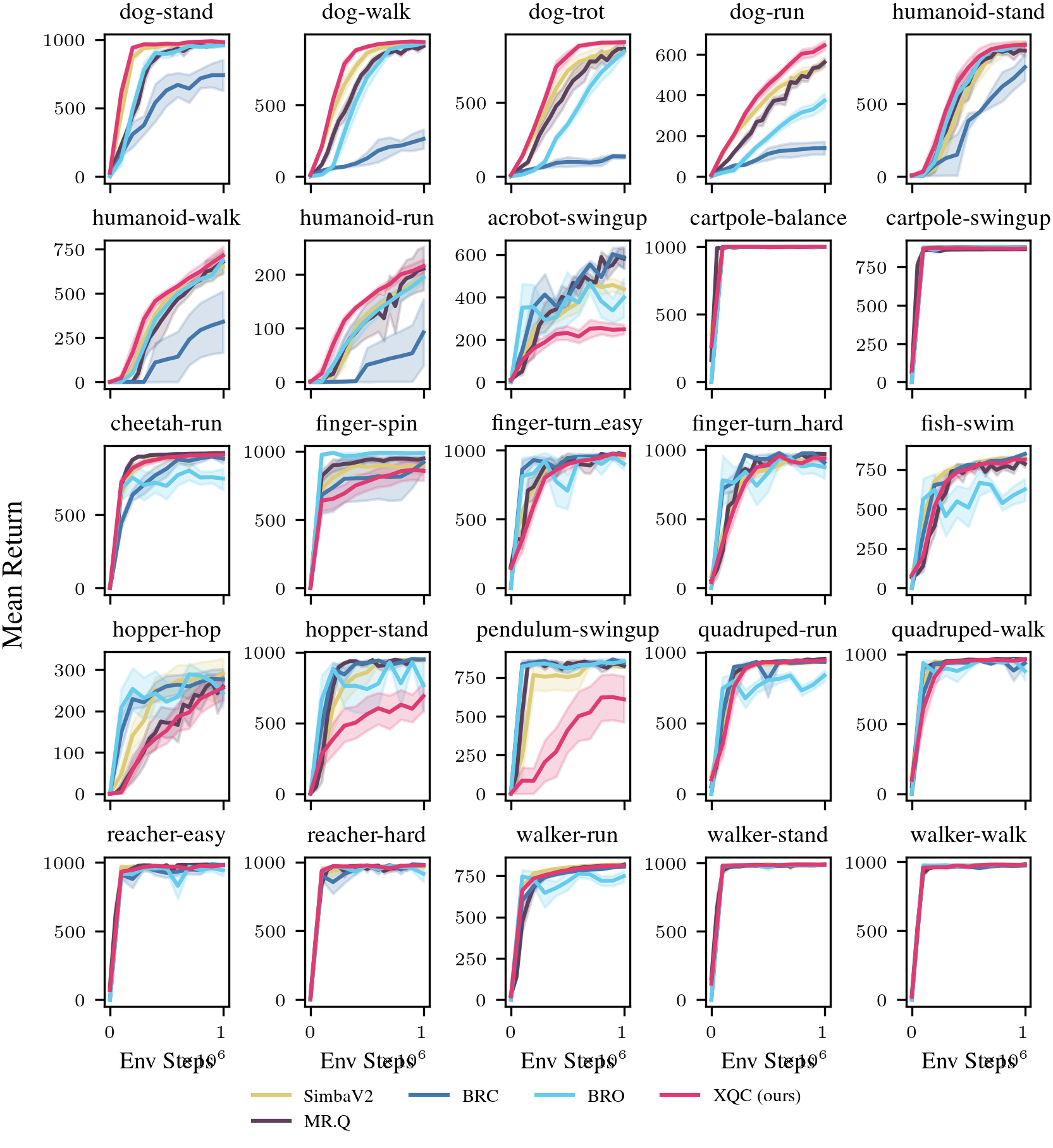}
    \caption{\XQC and baseline training curves for each of the \NEnvsDMC \DMC tasks. We show the mean and 90\% \SBCI aggregated per environment.}
\end{figure}

\clearpage
\subsection{MyoSuite}

\begin{figure}[h!]
    \centering
    \includegraphics[width=\linewidth]{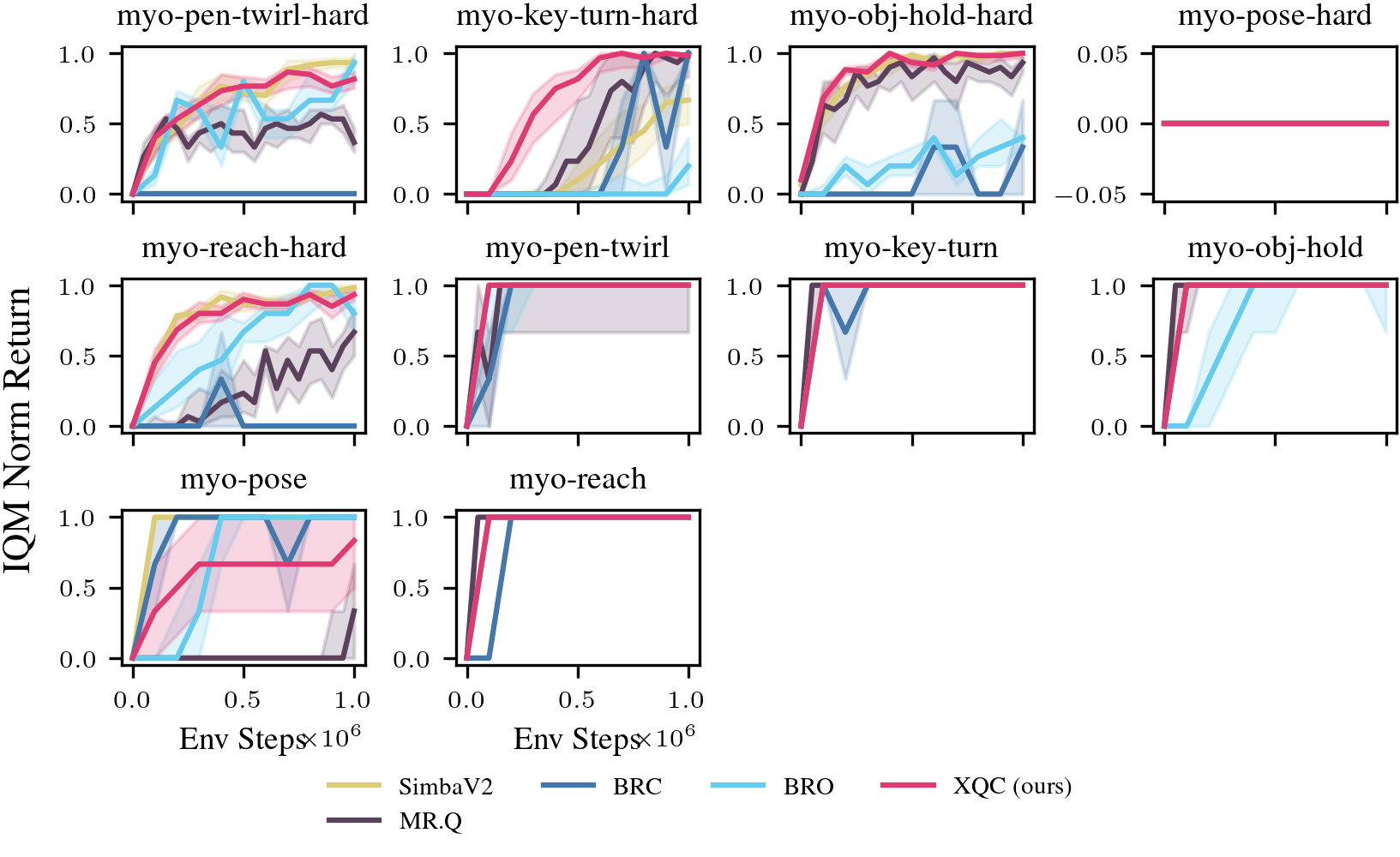}
    \caption{\XQC and baseline training curves for each of the \NEnvsMYO \Myo tasks. We show the \IQM and 90\% \SBCI aggregated per environment.}
\end{figure}

\begin{figure}[h!]
    \centering
    \includegraphics[width=\linewidth]{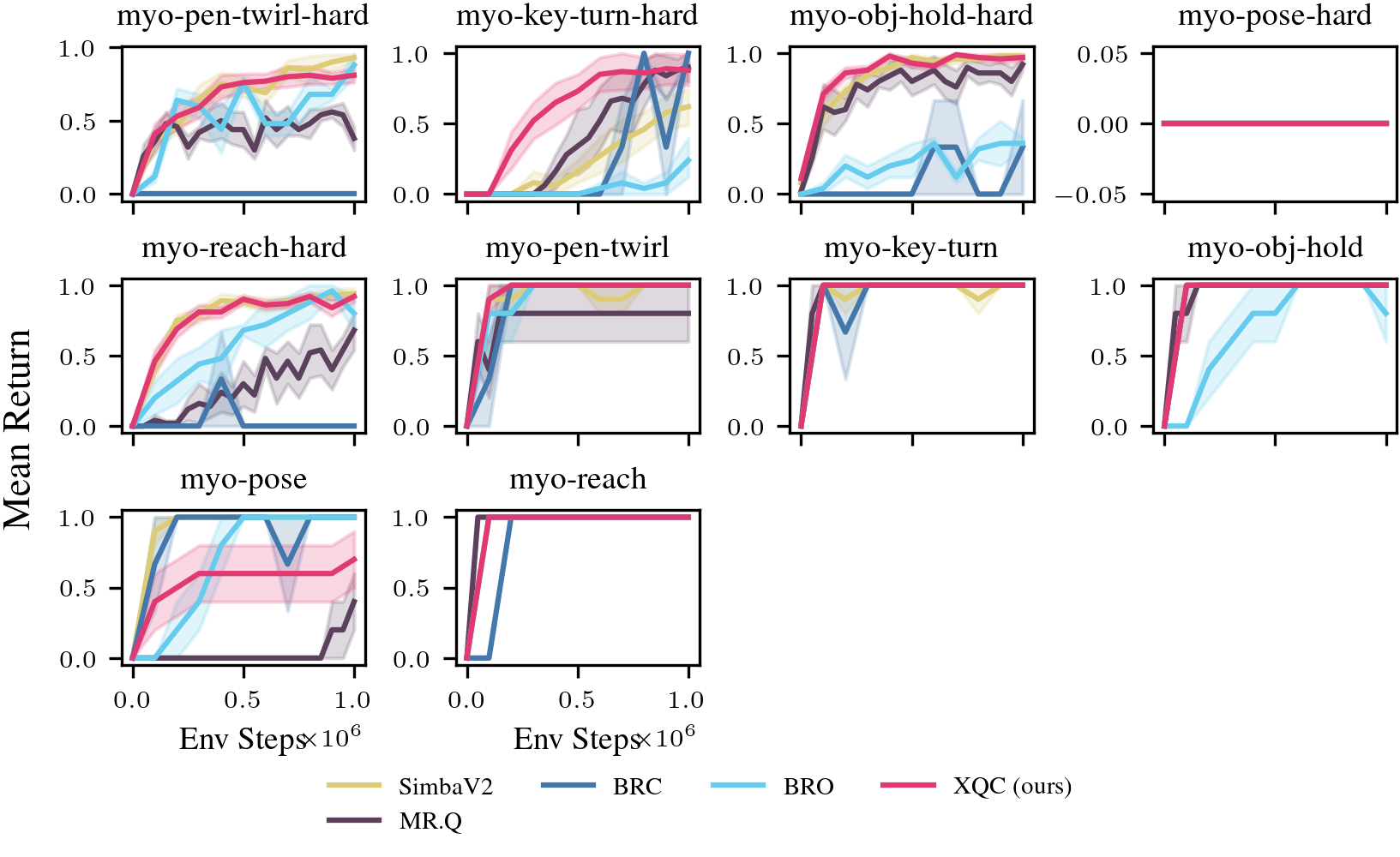}
    \caption{\XQC and baseline training curves for each of the \NEnvsMYO \Myo tasks. We show the mean and 90\% \SBCI aggregated per environment.}
\end{figure}

\clearpage
\subsection{MuJoCo Benchmark}
\begin{figure}[h!]
    \centering
    \includegraphics[width=\linewidth]{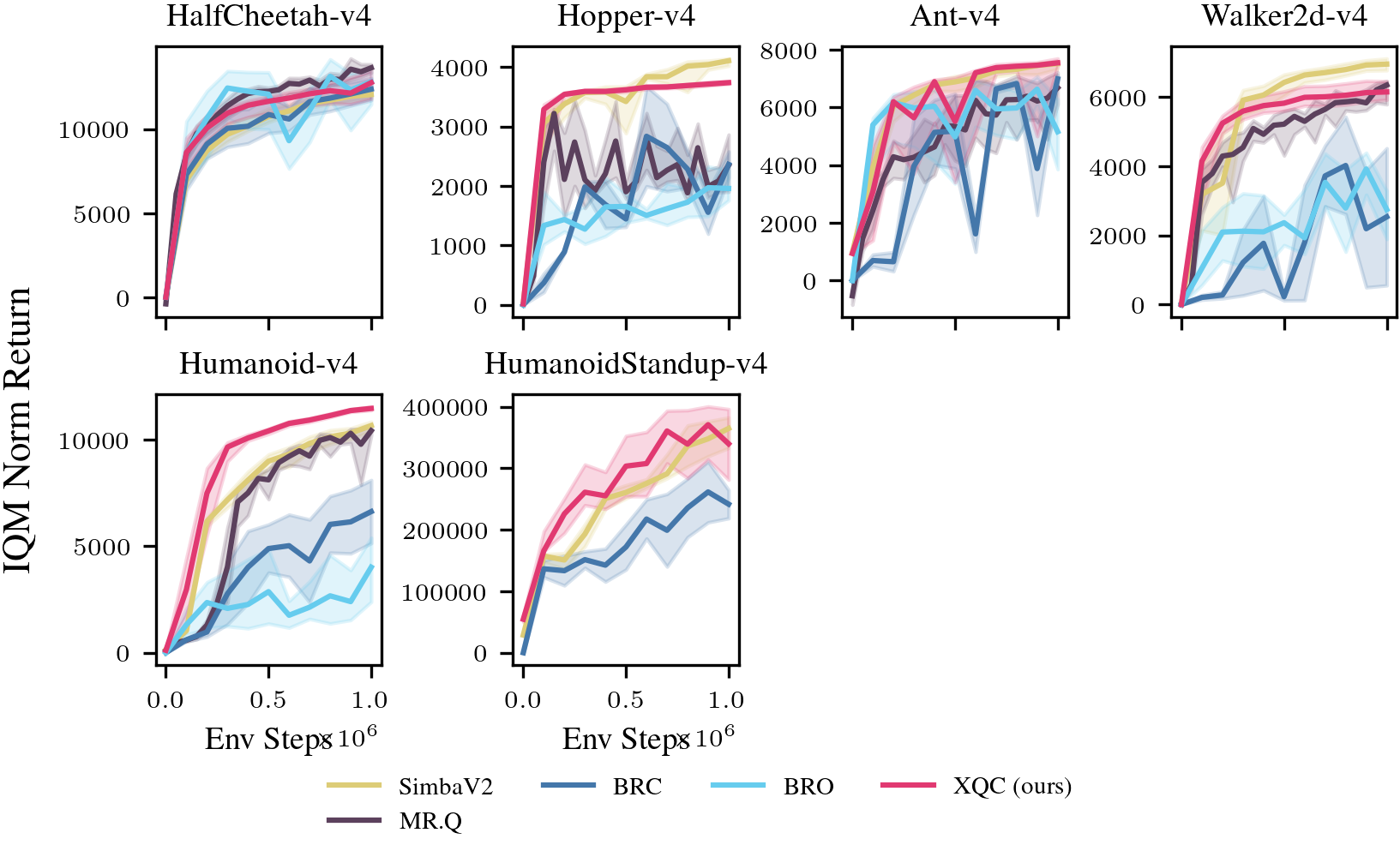}
    \caption{\XQC and baseline training curves for each of the \NEnvsMUJOCO \mujoco tasks. We show the \IQM and 90\% \SBCI aggregated per environment.}
\end{figure}

\begin{figure}[h!]
    \centering
    \includegraphics[width=\linewidth]{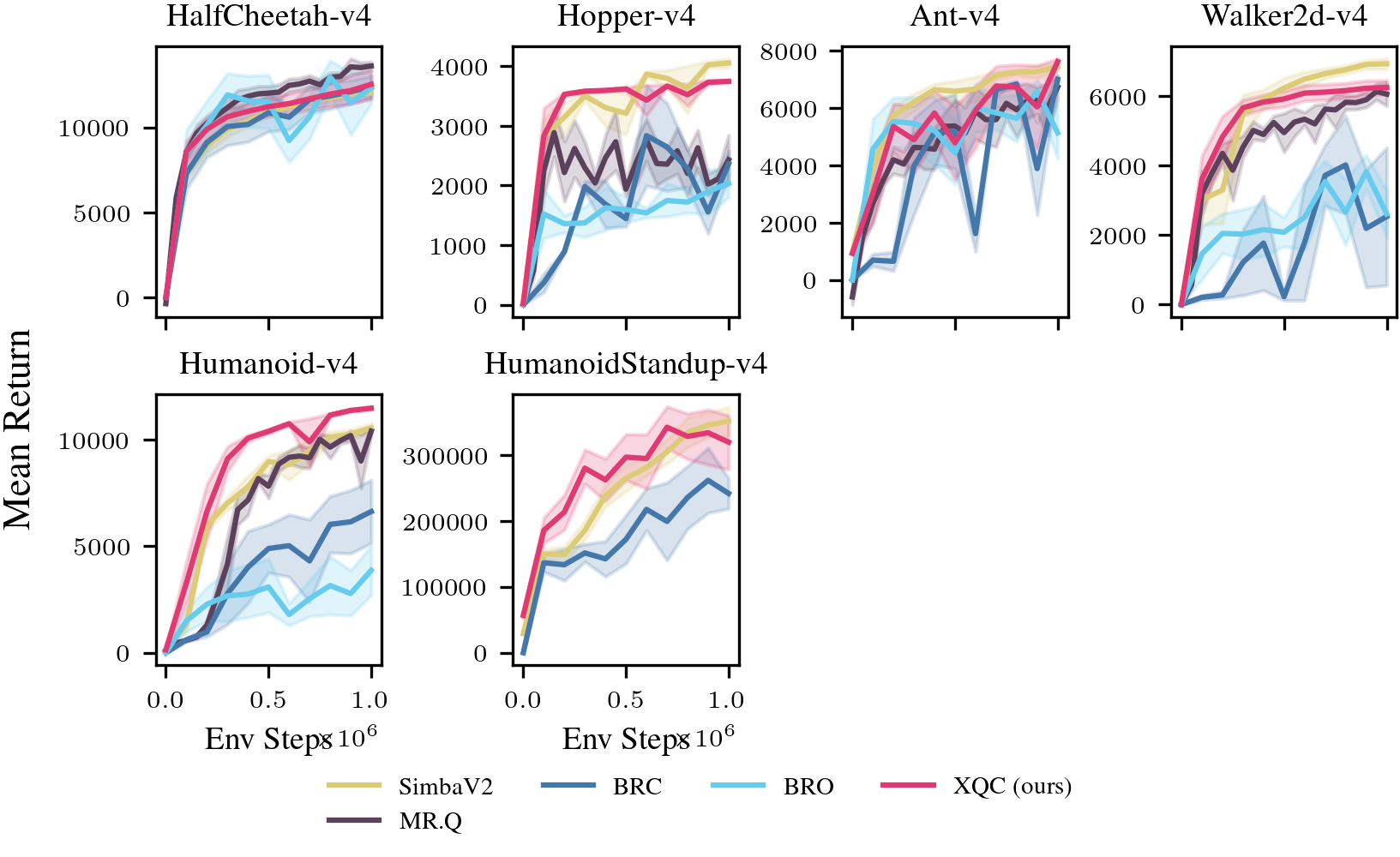}
    \caption{\XQC and baseline training curves for each of the \NEnvsMUJOCO \mujoco tasks. We show the mean and 90\% \SBCI aggregated per environment.}
\end{figure}

\newpage
\section{Plasticity Metrics}\label{app:plasticity}
\begin{figure}[h!]
    \centering
    \includegraphics[width=\linewidth]{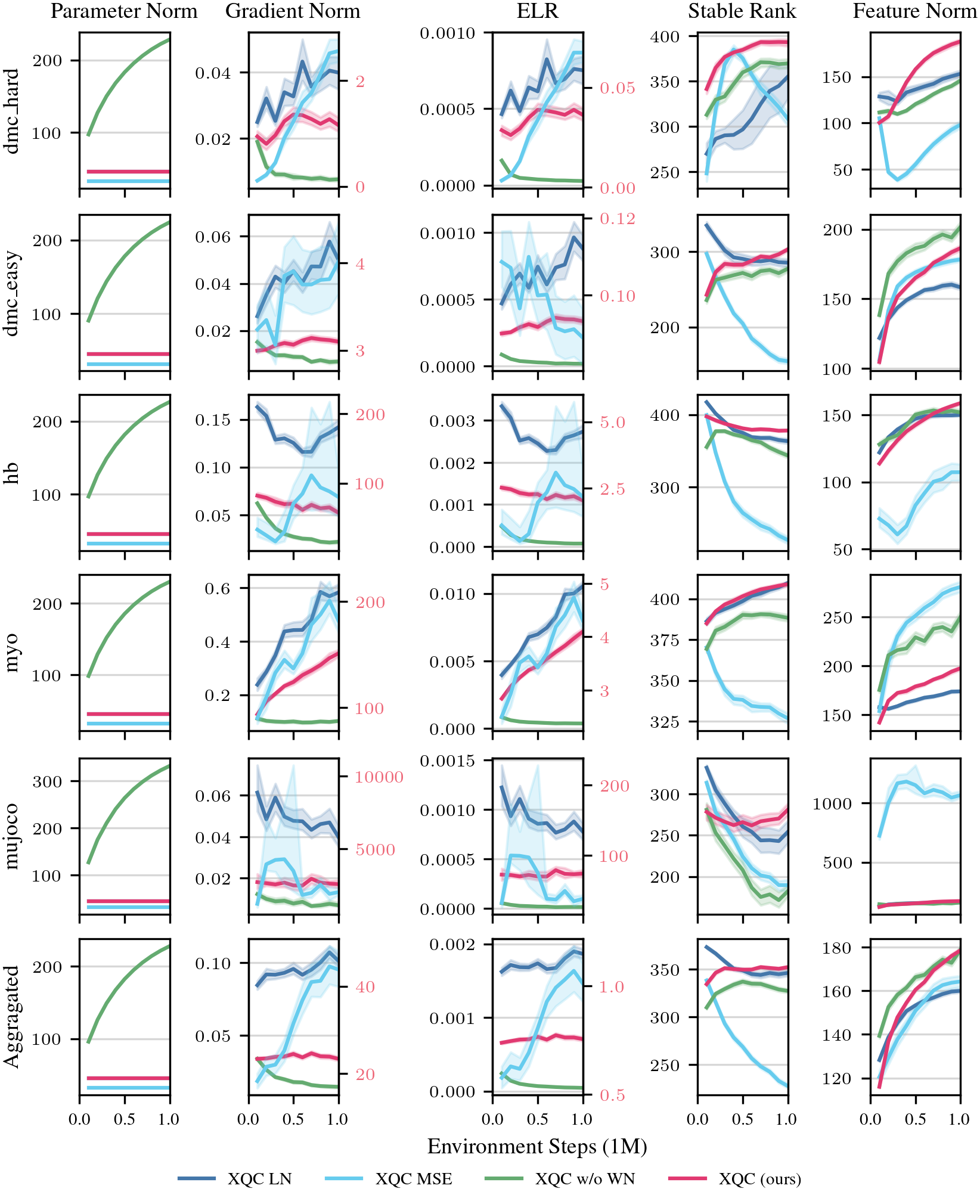}
    \caption{Per benchmark plasticity metrics for \XQC and architectural ablations.}
\end{figure}

\newpage
\section{Architecture ablations}\label{sec:ablations}
\begin{figure}[h!]
    \centering
    \includegraphics[width=\linewidth]{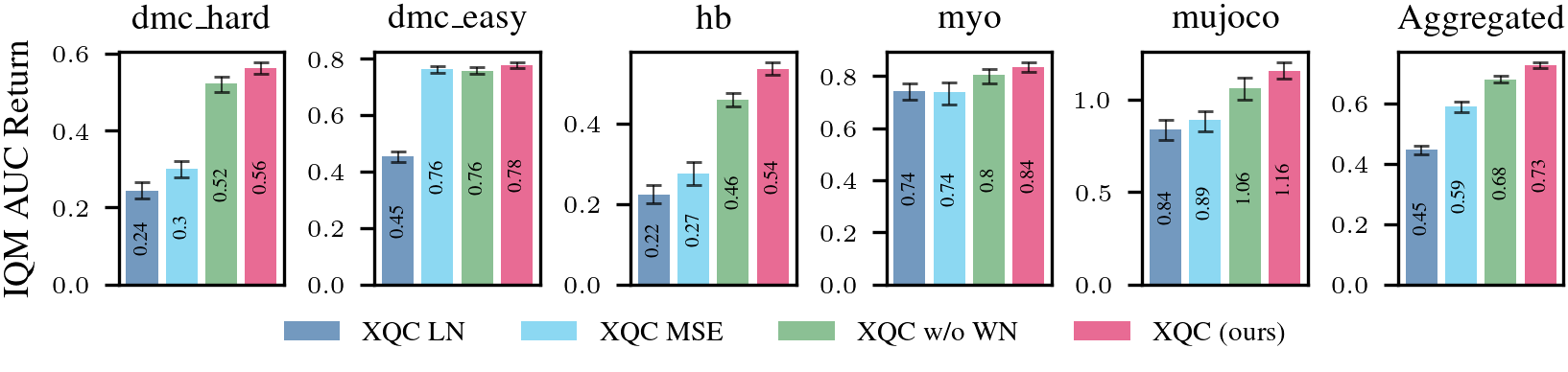}
    \caption{
    \textbf{Ablation study confirms the necessity of all three of \XQC's components.}
    We compare the full \XQC algorithm against three variants:
    one replacing \BN with \LN (\XQC \LN),
    one replacing the \CE loss with an \MSE loss (\XQC \MSE),
    and one without \WN (\XQC w/o \WN).
    Each component's removal results in a significant performance drop, demonstrating their synergistic contribution.
    }
    \label{fig:ablations}
\end{figure}

\newpage
\section{XQC UTD Scaling Training Curves}\label{app:utd_scaling}
\begin{figure}[h!]
    \centering
    \includegraphics[width=\linewidth]{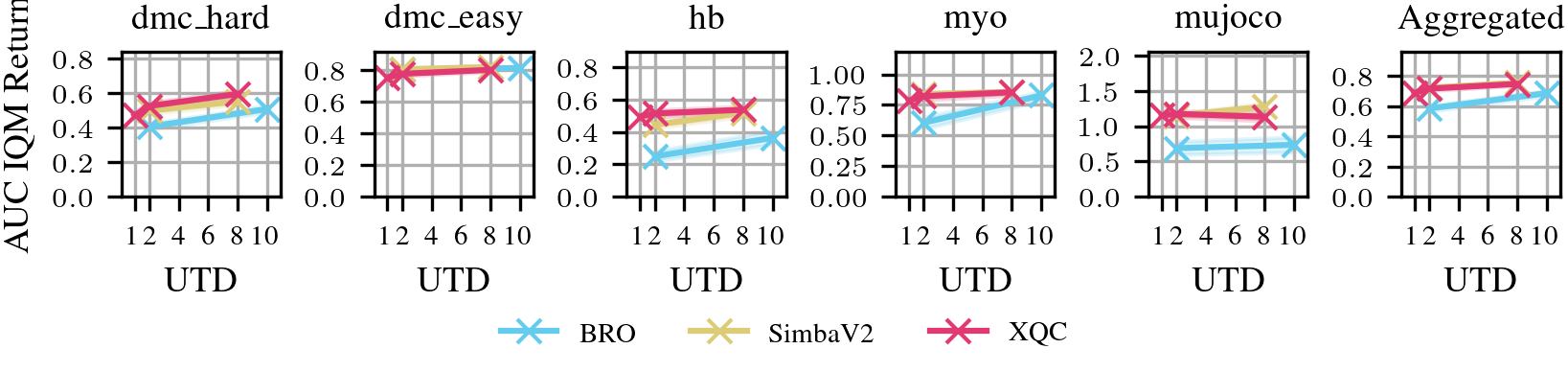}
    \caption{
    \textbf{\XQC stably improves with increased \UTD ratios.}
    We compare \IQM \AUC for \XQC trained with \UTD ratios $\in\{1,2,8,16\}$.
    Performance consistently improves with more updates, showcasing the stability of the learning process.
    }
\end{figure}
\begin{figure}[h!]
    \centering
    \includegraphics[width=\linewidth]{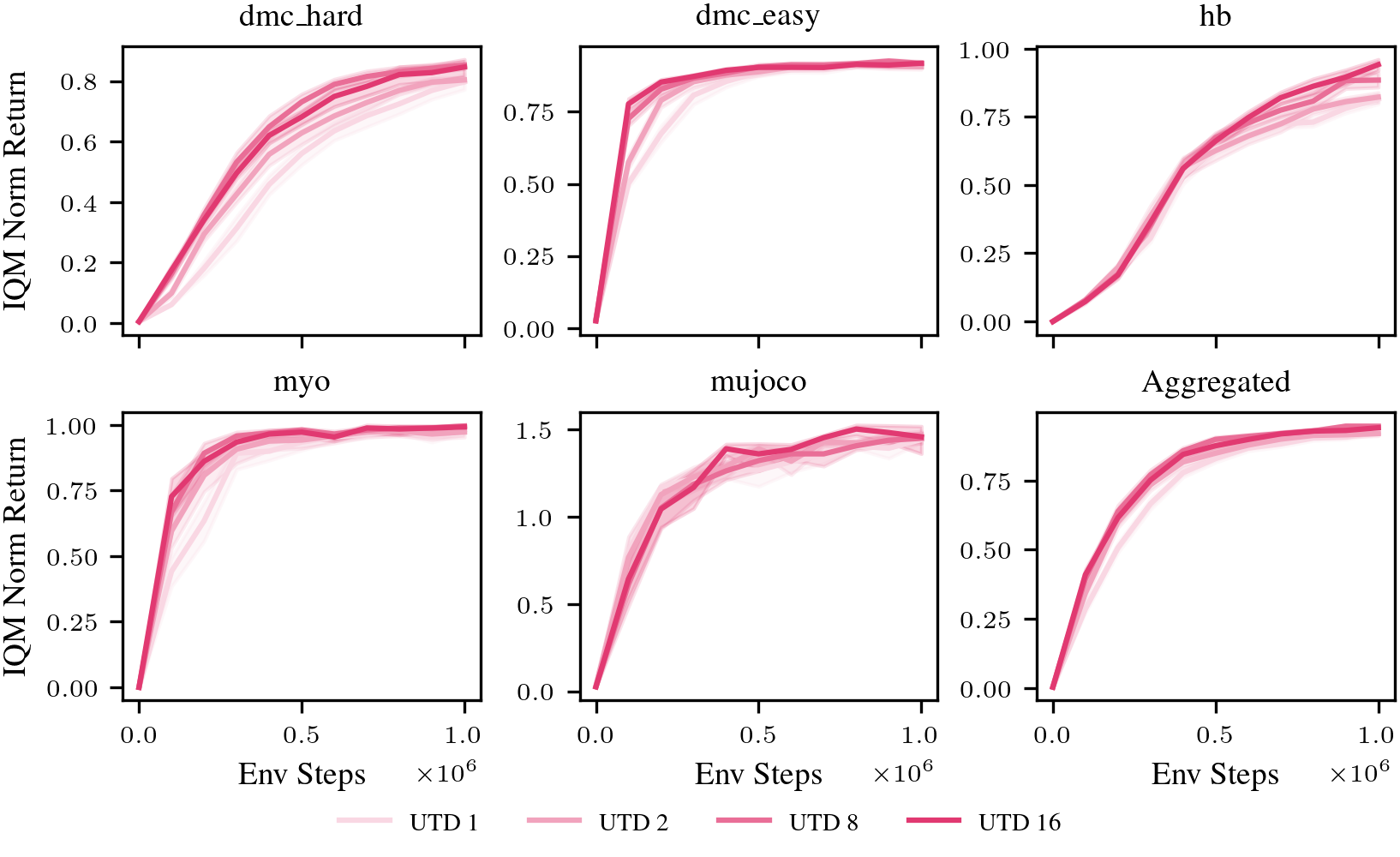}
    \caption{\UTD scaling. We present the area under the curve of the \IQM Norm Return. This measure captures fast and stable learning simultaneously.}
    \label{fig:scaling_utd}
\end{figure}

\clearpage
\section{Additional Conditioning Results}\label{sec:additional_conditioning}
Here we present additional conditioning results for alternative architectural components requested by the reviewers during the rebuttal.
We ablate combining \XQC with skip connections (with network depth in $\in\{4,6\}$, different batch sizes $\in\{128,256,512\}$ and weight decay $1e-4$.
Further, we investigate how depth scaling influences conditioning for dense \textsc{mlp}s with number of layers $\in\{2,4,6\}$.
The chosen architectural ablations on top of \XQC do not meaningfully influence conditioning or performance. The overall finding is in line with our main finding, that good conditioning correlates with good \RL performance.

\begin{figure}[h]
    \centering
    \includegraphics[width=\linewidth]{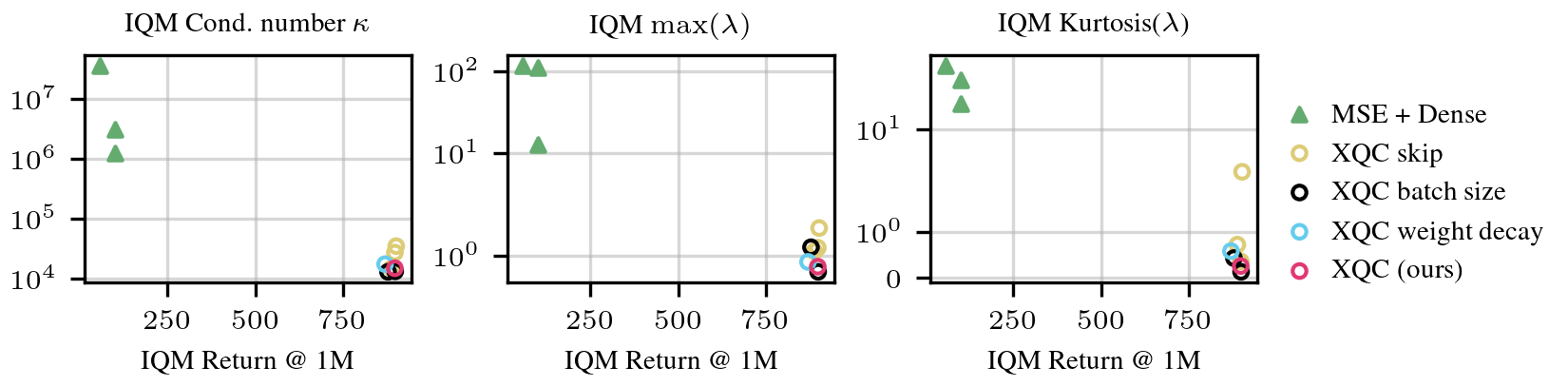}
    \caption{
    The condition numbers and maximum eigenvalues against the return at 1M steps on \DMC \texttt{dog-trot}.
    Architectures with lower condition numbers and lower maximum eigenvalues tend to have better final returns.
    }
    \label{fig:spectrum_summary_additional_ablations}
\end{figure}

\clearpage
\section{Theoretical analysis}
\label{app:theoretical_anaylsis}
The section details the proofs for bounding the gradient norms and Hessian condition numbers.
\begin{lemma}
\label{lem:chain_lipschitz}
For the loss $\gL(\vtheta, \gD) = l(\mY, \vf_\vtheta( \mX))$, if $\vf$ is $L_\vf$ Lipschitz in the L2 norm with respect to $\vtheta$, the L2 norm of the gradient has the following upper bound,
\begin{align}
    ||\nabla_\vtheta \gL(\vtheta,\vy, \vx)||_2 \leq L_\vf \cdot ||\nabla_\vf l(\vy, \vf_\vtheta( \vx))||_2.
\end{align}
\end{lemma}
\begin{proof}
Using the chain rule and the Cauchy-Schwarz inequality,
\begin{align}
    ||\nabla_\vtheta \gL(\vtheta,\vy, \vx)||_2 
    &\leq 
    ||\nabla_\vf l(\vy, \vf_\vtheta( \vx))||_2
    \cdot
    ||\nabla_\vtheta \vf_\vtheta( \vx))||_2
    \leq
    ||\nabla_\vf l(\vy, \vf_\vtheta( \vx))||_2
    \cdot L_\vf.
\end{align}
\end{proof}

\begin{proof}[Proof of Proposition \ref{prop:mse_grad}]
Standard calculus.
\end{proof}

\begin{proof}[Proof of Proposition \ref{prop:ce_grad}]
Standard calculus, and then using the difference of two categorical probability vectors (on a simplex) to bound the largest squared error as $2$.
\end{proof}

\begin{proof}[Proof of Theorem \ref{th:ce_elr}]
Combine Lemma \ref{lem:chain_lipschitz}, Proposition \ref{prop:ce_grad}, and the constrained parameter norm for the definition of the gradient update in Definition \ref{def:elr} to obtain an upper bound.
\end{proof}

\begin{lemma}
\label{lem:weyl}
For a symmetric matrix $\mA\in\sR^{m\times m}$ with ranked eigenvalues ${\lambda^A_1 \leq \dots \leq \lambda^A_m}$, then the eigenvalues of the sum of two such matrices $\mC=\mA + \mB$, then $\lambda^A_1 + \lambda^B_1\leq\lambda^C_1$ and 
$\lambda^A_m + \lambda^B_m \geq \lambda^C_m$.
This result holds for all finite sums.
\end{lemma}
\begin{proof}
Weyl's theorem applied to the sum of two Hermitian matrices \citep{bodmann2012matrix}. 
\end{proof}

\begin{proposition}
The mean squared error loss, $l(\vy, \hat{\vy}) = \frac{1}{2}||\vy -\hat{\vy}||_2^2$, $\vy \in \sR^d$ has a constant Hessian and therefore constant Hessian eigenvalues $\lambda$,
\begin{align}
    \nabla^2_{\hat{\vy}} l(\vy, \hat{\vy}) = \mI_d,\quad\lambda_{1:d} = 1.
\end{align}
\end{proposition}
\begin{proof}
Standard calculus.
\end{proof}

\begin{proposition}
The cross entropy loss, $l(\vy, \hat{\vy}) = -\sum_{i=1}^d y_{d}\log \hat{y}_{d}$ has the following Hessian and eigenvalue bounds given the model ${\hat{\vy} = \text{Softmax}(\vf_\vtheta(\vx))}$ where $y_i \geq \epsilon$,
\begin{align}
    \nabla_\vf^2 l(\vy, \vf_\vtheta( \vy)) = \mathrm{diag}(\hat{\vy}) - \hat{\vy}\hat{\vy}^\top, 
    \quad
    0 \leq \lambda_i \leq 1,
\end{align}
as $\sum_{i=1}^d y_i = 1, 0 \leq y_i \ \leq 1$.
The Hessian is singular due to the loss of degree-of-freedom in categorical probabilities.
\end{proposition}
\begin{proof}
Standard calculus.
\end{proof}

For Proposition \ref{prop:mse_hess} and \ref{prop:ce_hess}, we use that the objective's Hessian can be decomposed using the chain rule,
\begin{align}
    \nabla^2_\vtheta \gL(\vtheta,\vy, \vx)
    &=
    {\nabla_\vtheta \vf_\vtheta(\vx)}^\top
    \,
    \nabla^2_\vf l(\vy, \vf_\vtheta( \vx))
    \,
    \nabla_\vtheta \vf_\vtheta(\vx)
    +
    \nabla_\vf l(\vy, \vf_\vtheta(\vx) \nabla^2_\vtheta \vf_\vtheta( \vx),\notag\\
    &=
    \vg_\vtheta( \vx)^\top \mH_l(\vtheta,\vx, \vy)\vg_\vtheta(\vx)
    + \vg_l(\vtheta, \vx,\vy)^\top \mH_\vtheta( \vx).\label{eq:decomp_hessian}
\end{align}

\begin{proof}[Proof for Proposition \ref{prop:mse_hess}]
The first term of Equation \ref{eq:decomp_hessian} has a rank of 1 as it's an outer product, and $\vg(\vtheta,\vx)^\top\vg(\vtheta,\vx) = ||\vg||_2^2 \leq L_\vf^2$, so its eigenvalues  $\lambda_i \in [0, L_\vf^2]$.
Using Assumption \ref{ass:hess}, the eigenvalues of the second term are bounded by $[-m |g|_{\text{max}} \lambda_m^f, m |g|_{\text{max}} \lambda_m^f]$.
As the gradient elements cannot be upper bounded (i.e., Proposition \ref{prop:mse_grad}), the Hessian of the loss has eigenvalue range 
$[\mu^2 - 2m|g|_\text{max}\lambda_m^f, \mu^2 + 2m|g|_\text{max}\lambda_m^f]$,
which leads to an unbounded condition number due to both the largest eigenvalue $\rightarrow \infty$ and the the case that the smallest eigenvalue is $0$ when adding eigenvalues from both terms due to Weyl's theorem (Lemma \ref{lem:weyl}).
\end{proof}

\begin{proof}[Proof for Proposition \ref{prop:ce_hess}]
The first term in Equation \ref{eq:decomp_hessian} has eigenvalue bounds $[0, L_\vf^2]$ (see previous proof).
It's positive semi-definite so we know 0 is a lower bound on the eigenvaluse.
Since the max eigenvalue is non-zero, we know the Frobenius norm of $\vg$ is greater or equal to than the trace of the outer product, and the trace is also the sum of eigenvalues, so we can bound the (largest) non-zero eigenvalue by $L_\vf^2$. 
The second term in Equation \ref{eq:decomp_hessian} has range $\lambda_i \in [-2 \lambda_m^f, 2 \lambda_m^f]$, as they are bounded by
$[-\sum_i |g_i| \lambda_m^H, \sum_i |g_i| \lambda_m^H]$
and $0\geq |g_i|\geq 1$, $\sum_i g_i = 0$.
With Weyl's theorem (Lemma \ref{lem:weyl}) we have
$\lambda \in[\mu^2 - 2\lambda_m^H, \mu^2 + 2 \lambda_m^H + L_\vf^2]$.
If $\mu^2 = 2\lambda_m^H + \epsilon$, then we have
$\lambda \in[\epsilon , 4\lambda_m^H + \epsilon + L_\vf^2]$, so
$$
\kappa \leq \frac{4\lambda_m^f + L_\vf^2 + \epsilon}{\epsilon}
$$
which concludes the proof.
Unsurprisingly, the upper bound is only finite if the regularization ensures positive definiteness of the objective's Hessian.
\end{proof}

\end{document}